\documentclass[twoside,11pt]{article}

\usepackage{jmlr2e}

\usepackage{hyperref}
\usepackage{url}
\usepackage{graphicx} % more modern

\usepackage{algorithm}
\usepackage{algorithmic}

\newcommand{\uh}{\ensuremath{\hat{u}}}
\newcommand{\vh}{\ensuremath{\hat{v}}}
\newcommand{\wh}{\ensuremath{\hat{w}}}
\newcommand{\us}{\ensuremath{u^\ast}}
\newcommand{\vs}{\ensuremath{v^\ast}}
\newcommand{\ws}{\ensuremath{w^\ast}}
\newcommand{\Uc}{\ensuremath{\mathcal{U}}}
\newcommand{\Vc}{\ensuremath{\mathcal{V}}}
\newcommand{\Wc}{\ensuremath{\mathcal{W}}}

\newcommand{\fu}{\ensuremath{f_{\mathrm{util}}}}
\newcommand{\fp}{\ensuremath{f_{\mathrm{priv}}}}
\firstpageno{1}

\begin{document}

%\title{A Learning Approach to Preserve Privacy
%\\from Inference Attacks}
\title{Minimax Filter: Learning to Preserve Privacy\\
from Inference Attacks}

\author{\name Jihun Hamm \email hammj@cse.ohio-state.edu \\
       \addr Department of Computer Science and Engineering\\
       The Ohio State University\\
       Columbus, OH 43210, USA
       %\AND
       %\name Michael I.\ Jordan \email jordan@cs.berkeley.edu \\
       %\addr Division of Computer Science and Department of Statistics\\
       %University of California\\
       %Berkeley, CA 94720-1776, USA
       }

\editor{}

\maketitle

\begin{abstract}%   <- trailing '%' for backward compatibility of .sty file
Preserving privacy of continuous and/or high-dimensional data such as 
images, videos and audios, can be challenging with 
syntactic anonymization methods which are designed for discrete attributes.
Differentially privacy, which uses a more rigorous definition of privacy loss, 
has shown more success in sanitizing continuous data. 
However, both syntactic and differential privacy are susceptible to 
inference attacks, i.e., an adversary can accurately infer sensitive
attributes from sanitized data. 
The paper proposes a novel filter-based mechanism which preserves privacy
of continuous and high-dimensional attributes against inference attacks.
Finding the optimal utility-privacy tradeoff is formulated as
a min-diff-max optimization problem. 
The paper provides an ERM-like analysis of the generalization error
and also a practical algorithm to perform minimax optimization. 
In addition, the paper proposes a noisy minimax filter which combines
minimax filter and differentially-private mechanism. Advantages
of the method over purely noisy mechanisms is explained and demonstrated with examples.
Experiments with several real-world tasks including facial expression classification,
speech emotion classification, and activity classification from motion, show that
the minimax filter can simultaneously achieve similar or higher target task accuracy 
and lower inference accuracy, often significantly lower than previous methods.
\end{abstract}


%%%%%%%%%%%%%%%%%%%%%%%%%%%%%%%%%%%%%%%%%%%%%%%%%%%%%%%%%%%%%%%%%%%%%%%%%%%%%%%%%%%%%%%%%%%
\section{Introduction}
%%%%%%%%%%%%%%%%%%%%%%%%%%%%%%%%%%%%%%%%%%%%%%%%%%%%%%%%%%%%%%%%%%%%%%%%%%%%%%%%%%%%%%%%%%%

Privacy is an important issue when data collected from or related to individuals
are analyzed and released to a third party. 
In response to growing privacy concerns,
various methods for privacy-preserving data publishing have been proposed 
(see \citealt{Fung:2010:CSUR} for a review.) 
{\it Syntactic anonymization} methods, such as $k$-anonymity
 \citep{Sweeney:2002:IJUFKS} and $l$-diversity \citep{Mach:2007}
focus on anonymization of quasi-identifiers and protection of sensitive attributes
in static databases. 
However, it is known that syntactic anonymization is susceptible to 
several types of attacks such as the DeFinetti attack \citep{Kifer:2009}. 
An adversary may be able to accurately {\it infer} sensitive attributes 
of individuals from insensitive, sanitized attributes. 
High-dimensional data also poses a challenge for syntactic anonymization methods. 
For example, $k$-anonymity is known to be ineffective for high-dimensional 
sparse databases \citep{Narayanan:2008}.
In addition, syntactic anonymization methods are designed with discrete attributes in mind. 
However, continuous attributes such as videos, images, or audios cannot be discretized
by binning or clustering without loss of information. 
Besides, conventional categorization of attributes as 
identifiers, quasi-identifiers, or sensitive information becomes
ambiguous with multimedia-type data.
For example, an image can be an identifier if it contains the face of the data owner. 
However, even if the face is blurred, other attributes considered sensitive by the
owner such as gender or race can still be recognizable. 
Furthermore, identifying or sensitive information can be revealed through correlation
with other information such as the background or other people in the scene. 
 
% A bit more details
{\it Differential privacy} \citep{Dwork:2004:CRYPTO,Dwork:2006:TC,Dwork:2006:ALP}
was proposed to address many weaknesses of syntactic methods (see
the discussion by \citealt{Clifton:2013}.)
Differential privacy has a more formal privacy guarantee than that of
syntactic methods, and is applicable to many problems beyond database release
\citep{Dwork:2014}.
In particular, differential privacy can be defined for
continuous and/or high-dimensional attributes as well as for functions  \citep{Hall:2013:JMLR}.
However, similarly to syntactic anonymization, differential privacy is not immune 
to inference attacks \citep{Cormode:2011}, as differentially privacy
only prevents an adversary from gaining {\it additional} knowledge by
inclusion/exclusion of a subject \citep{Dwork:2014},
and not from gaining knowledge from released data itself.
Therefore, an adversary may still guess sensitive attributes 
of subjects from differentially-private attributes with confidence.

To preserve privacy of continuous high dimensional data from inference attacks,
this paper proposes an {approach} which differs significantly
from previous syntactic or differentially-private approaches. 
Consider a scenario where a social media user wants to obfuscate all faces 
in her picture with minimal distortion before posting the picture online. 
The obfuscation mechanism proposed in the paper is a type of {\it filtering} of 
the original features by a non-invertible transformation. 
How to choose an optimal filter is explained in the following general description.
Once the filtered data (e.g., obfuscated pictures) are released, 
an adversary will try to infer sensitive or identifying attributes from the data
in particular using machine learning predictors. 
Therefore the data owner needs a filter that can minimize the maximum accuracy that any
adversary may achieve in predicting the sensitive or identifying attributes.
This is an instance of {\it minimax games} between the data owner and the adversary. 
The privacy of filtered data is measured by 
the {\it expected risk} of adversarial algorithms on specific inference tasks such as
identification.
However, if privacy is the only goal, near-perfect privacy is achievable with 
a simple mechanism that sends no or garbage data, which has no utility for 
any party.
To avoid those trivial solutions, (dis)utility of filtered data needs to be
considered as the second goal. Disutility can be measured by the amount of 
distortion of the original data after filtering. Alternatively, 
if there are particular tasks of interest to be performed on the data by 
non-adversarial analysts, then again the expected risk of the target tasks 
can be used as disutility. 
The two goals---achieving privacy and utility---are often mutually conflicting,
and finding an optimal tradeoff between the two is a central question
in privacy research (see Related work.)
This paper proposes to minimize the difference of two risks by the minimax optimization
of (\ref{eq:joint goal 1}). The solution to the optimization problem will be referred to as 
{\bf minimax filter}. 
In the literature, several methods have been used to solve 
continuous minimax optimization problems, 
including the method by \cite{Kiwiel:1987} used in \cite{Hamm:2015a}. 
The paper uses a simpler optimization method based on the classic theorem
of \cite{Danskin:1967}.

A notable assumption this paper makes is that the training data for computing 
an optimal filter are independent of the test data. 
For example, there are publicly available data sets
that can be used to compute minimax filters such as those from the UCI data repository.\footnote{\url{http://archive.ics.uci.edu/ml/}} 
As the training data set is already public information, the paper considers only the 
privacy of the subjects who use the filter at test time. A similar assumption was made 
in \cite{Hamm:2016} for knowledge transfer purposes. 
After the filter is learned from training data, a new test subject can use the 
filter to obfuscate her data by herself without requiring a third party to collect 
and process her raw data.
Note that this setting is analogous to the setting of local differential privacy \citep{Duchi:2013}
where the entity that collects data is not trusted.
Since the training procedure can only access empirical risks, the performance
of the filter on test data is given in the form of expectation/probability.
The paper presents an analysis of generalization error for empirical minimax optimizers
in analogy with the analysis of empirical risk minimizers (ERM).

% Differential privacy by noisy filter
The goal of minimax filter is to prevent inference attacks, and its
privacy guarantee is quite different from those of other privacy mechanisms.
It is task-dependent and is given in probability or expectation rather than 
given absolutely, which may be considered weaker than other privacy criteria such as
differential privacy.
Furthermore, the sanitized data whose sensitive information is filtered out
may become unsafe in the future if people's perception of which attribute is sensitive
changes over time.
Since the goal of minimax filter and the goal of differential privacy are very different, 
it is natural to ask if the two methods can be combined to take advantages
of both methods.
Consequently, this paper presents an extension of minimax filter called 
{\bf noisy minimax filter},
which combines the filter with additive noise mechanism to satisfy
the differential privacy criterion. 
Two methods of combination---preprocessing and postprocessing---are 
proposed (see Fig.~\ref{fig:preprocessing}.)
In the preprocessing approach, minimax filter is applied {before} perturbation
to reduce the sensitivity of transformed data, so that the same level
of differential privacy is achieved with less noise. 
Similar ideas have been utilized before, where data are transformed by 
Discrete Fourier Transform \citep{rastogi2010differentially} and by Wavelet Transform \citep{xiao2011differential} before noise is added.
In the postprocessing approach, minimax filter is applied {after} perturbation,
and its performance is compared with the preprocessing approach.
%However, it requires a trusted curator who trains the minimax filter.
%In the postprocessing approach, minimax filter is applied {after} perturbation.
%This approach may not benefit from reduced sensitivity as in the preprocessing approach,
%but it can be used to  of not requiring a trusted curator and not leaking 
%information through the released filter. 

% Brief discussion of experiments
Minimax filter and its extensions are evaluated with several real-world tasks 
including facial expression classification,
speech emotion classification, and activity classification from motion.
Experiments show that publicly available continuous and high-dimensional 
data sets are surprisingly susceptible to subject identification attacks,
and that minimax filters can reduce the privacy risks to near chance levels
without sacrificing utility much.
Experiments with noisy minimax filter also yield intuitive results.
Differential privacy and resilience to inference attack are indeed different
goals, such that using differentially private mechanism alone to achieve
the latter requires a large amount of noise that destroys utility of data. 
In contrast, minimax filters can suppress inference attack with little
loss of utility with or without perturbation. 
Therefore, adding a small amount of noise to the minimax filter
can provide a formal differential privacy to a degree and also high 
on-average task-dependent utility and privacy against inference attacks.

To summarize, the paper has the following contributions.
%\vspace{-0.1in}
\begin{itemize} \setlength{\itemsep}{-2pt}

\item The paper proposes a novel filtering approach which preserves privacy
of continuous and high-dimensional attributes against inference attacks.
This mechanism is different from previous mechanisms in many ways; 
in particular, it is a learning-based approach and is task-dependent. 
\item The paper measures utility and privacy by expected risks, and 
formulates the utility-privacy tradeoff as a min-diff-max optimization problem.
The paper provides an ERM-like analysis of the generalization performance of
empirical optimizers.
\item The paper presents a practical algorithm which can find minimax filters for
a broad family of filters and losses/classifiers. 
The proposed optimization algorithm and supporting classes can be found on the 
open-source repository.\footnote{\url{https://github.com/jihunhamm/MinimaxFilter}}
\item The paper proposes preprocessing and postprocessing approaches 
to combine minimax filter with noisy mechanisms. The resulting combination
can achieve resilience to inference attacks and differential privacy at the same time.
\item The paper evaluates proposed algorithms on real-world tasks
and compares them with representative algorithms from the literature.
\end{itemize}

% Organization.
The rest of the paper is organized as follows.
Sec.~\ref{sec:related work} presents related work in the literature. 
Sec.~\ref{sec:minimax filter} presents minimax filters and
analyzes its generalization performance on test data.
Sec.~\ref{sec:optimization} explains the difficulty of solving general minimax
problems, and present a simple alternating optimization algorithm.
Sec.~\ref{sec:noisy minimax filter} presents noisy minimax filters and two
types of perturbation by additive noise.
Sec.~\ref{sec:experiments} evaluates minimax filters with three data sets 
compared to non-minimax approaches and also evaluates noisy minimax filters under
various conditions. 
Sec.~\ref{sec:conclusion} concludes the paper with discussions.

%%%%%%%%%%%%%%%%%%%%%%%%%%%%%%%%%%%%%%%%%%%%%%%%%%%%%%%%%%%%%%%%%%%%%%%%%%%%%%%%%%%%%%%%%%%%%%%%%%%
\section{Related work}\label{sec:related work}
%%%%%%%%%%%%%%%%%%%%%%%%%%%%%%%%%%%%%%%%%%%%%%%%%%%%%%%%%%%%%%%%%%%%%%%%%%%%%%%%%%%%%%%%%%%%%%%%%%%

Optimal utility-privacy tradeoff is one of the main goals in privacy research.
Utility-privacy tradeoff has particularly been well-studied 
under differential privacy assumptions
\citep{Dwork:2004:CRYPTO,Dwork:2006:TC,Dwork:2006:ALP},
in the context of the statistical estimation 
\citep{Smith:2011, Alvim:2012, Duchi:2013} and learnability \citep{Kasiviswanathan:2011}. 
Other measures of privacy and utility were also proposed.
Information-theoretic quantities were proposed by
\cite{Sankar:2010,Rebollo:2010,Calmon:2012} who analyzed privacy in terms of
the rate-distortion theory in communication.
One problem with using mutual information or related quantity is
that it is difficult to estimate mutual information of high-dimensional and
continuous variables in practice without assuming a simple distribution.
In contrast, this paper proposes to use classification or regression risks 
to measure privacy and utility,
which is directly computable from data without making assumptions on the distribution.
% Classification-based measure. 
Regarding the use of risks in this paper, 
classification error-based quantities have been suggested in the literature \citep{Iyengar:2002,Brickell:2008,Li:2009}.
However, privacy in those works is measured either by syntactic anonymity
or probabilistic divergence which are mainly suitable for discrete attributes. 
In this paper, privacy and utility are both defined using risks and are therefore
directly comparable when defining the tradeoff of the two.
Furthermore, the proposed method explicitly preserves privacy against
inference attacks, which both syntactic and differentially-private methods
are known to be susceptible to \citep{Cormode:2011}.

Most of the aforementioned works focused on the analyses of utility-privacy tradeoff 
using different measures and assumptions. Few studied efficient algorithms to
actively find optimal tradeoff which this paper aims to do.
For discrete variables, \cite{Krause:2008} showed the NP-hardness of 
optimal utility-privacy tradeoff in discrete attribute selection,
and demonstrated near-optimality of greedy selection. In particular,
they used a weighted difference of utility and privacy cost as the joint cost
similar to this work. 
\cite{Ghosh:2009} proposed geometric mechanism and linear programming
to achieve near-optimal utility for unknown users.
Note that the optimization problems with discrete distributions are quite different
from the problems involving high-dimensional and/or continuous distributions.

Algorithms for preserving privacy of high-dimensional face images
has been proposed previously. 
%\cite{Newton:2005,Gross:2008,Enev:2012,Whitehill:2012}:
\cite{Newton:2005} applied k-anonymity to images;
\cite{Enev:2012} proposed to learn a linear filter using Partial Least Squares
to reduce the covariance between filtered data and private labels;
\cite{Whitehill:2012} also proposed a linear filter using
the log-ratio of the Fisher's Linear Discriminant Analysis metrics.
\cite{xu2017cleaning} presented a related method of preserving privacy of 
linear predictors using the Augmented Fractional Knapsack algorithm.
This paper differs from these in several aspects:
it is not limited to linear filters and is applicable to 
arbitrary differentiable nonlinear filters such as multilayer neural networks;
it directly optimizes the utility-privacy risk instead of optimizing heuristic
criteria such as covariance differences or LDA log-ratios. 

The noisy minimax filter proposed in Sec.~\ref{sec:noisy minimax filter}
bears a resemblance to the work of \cite{rastogi2010differentially} and \cite{xiao2011differential}. 
\cite{rastogi2010differentially} presented a differentially private method 
of answering queries on time-series data. They used Discrete Fourier
Transform to reduce the data dimension and homomorphic encryption to perform
distributed noise addition which outperformed the naive noise addition method.
\cite{xiao2011differential} presented a differentially private range-counting 
method where they used wavelets to transform the data before adding noise. 
Effectiveness of the method was analyzed and also demonstrated empirically.
The noisy minimax filter presented in this paper, especially the preprocessing approach,
is similar in concept to those works in that the combination of data transformation 
and perturbation is used to enhance utility. 
However, the transform in this paper (i.e., the minimax filter)
is learned from data for specific tasks unlike the Fourier or the Wavelet transform
which are data and task independent. 

Lastly, the alternating optimization algorithm (Alg.~\ref{alg:alternating})
presented in this paper is related to the algorithm proposed by \cite{Goodfellow:2014}. 
The algorithm proposed in this paper solves a min-diff-max problem to find 
an optimal utility-privacy tradeoff, while \cite{Goodfellow:2014} solve
a minimax problem to learn generative models.
%\item \citep{LeFevre:2006} Workload-aware anonymization

Parts of this paper have appeared in conference proceedings
 \citep{Hamm:2015a,Hamm:2017a}. New materials in this paper include 
reformulations of concepts and terms, ERM-like analysis of generalization error, 
new closed-form examples for minimax optimization, and
an alternating optimization algorithm to solve minimax problems. 
%All results in this paper are produced using the new algorithm. 

%%%%%%%%%%%%%%%%%%%%%%%%%%%%%%%%%%%%%%%%%%%%%%%%%%%%%%%%%%%%%%%%%%%%%%%%%%%%%%%%%%%%%%%%%%%
\section{Minimax Filter}\label{sec:minimax filter}
%%%%%%%%%%%%%%%%%%%%%%%%%%%%%%%%%%%%%%%%%%%%%%%%%%%%%%%%%%%%%%%%%%%%%%%%%%%%%%%%%%%%%%%%%%%

In this section, minimax filter is introduced and discussed in detail, 
and its generalization error is analyzed. 

\subsection{Formulation}
Minimax filter is a non-invertible transformation of raw features/attributes 
such that the transformed data have optimal utility-privacy tradeoff.
Non-invertibility is assumed so that original features are not always
recoverable from the filtered data.
Let's assume the filter is deterministic; randomize filters will be discussed in
Sec.~\ref{sec:noisy minimax filter}. 
Let $\mathcal{X} \subset \mathbb{R}^D$ be the space of features/attributes
as real-valued vectors. Note that discrete attributes can
also be represented by real vectors, e.g., by one-hot vector. 
Let the filter be a map
\begin{equation}
g(x;u)\in G: \mathcal{X}\times \Uc \to \mathbb{R}^d
\end{equation}
which is continuous in $x$ and is continuously differentiable w.r.t. the 
parameter $u\in \Uc$.
Given a filtered output $g(x)$, an adversary can make a prediction
$h_p(g(x);v)$ of a private variable $y$ which can be an identifying
or sensitive attribute. 
The prediction function $h_p(g(x);v)$ parameterized by $v\in \Vc$ 
is also assumed to be continuous in $v$ and continuously differentiable w.r.t. to the input  $g(x)$.
The paper proposes to use expected risk to measure the privacy of filtered output
against adversarial inference:
\begin{equation}\label{eq:privacy risk}
\fp(u,v) \triangleq E[l_p(h_p(g(x;u);v), y)],
\end{equation}
where $l_p(\cdot)$ is a continuously differentiable loss function.
From the assumptions above, $\fp$ is continuously differentiable
w.r.t. the filter parameter $u$. 

% About utility-privacy tradeoff
Trivial solutions to maximize privacy already exist, which are the
filters that output random or constant numbers independent of actual data. 
However, such filters have no utility whatsoever for any party. 
To avoid such trivial solutions, it is necessary to consider the secondary goal of 
maximizing utility. %minimizing the disutility.

Suppose the disutility of filtered data is measured by the distortion of 
the original data.
If $g(x;u)$ is the filter/encoder $\mathcal{X}\to \mathbb{R}^d$, 
then one can construct the decoder $h(\cdot\;;w):\mathbb{R}^d \to \mathcal{X}$,
such that following reconstruction error
\begin{equation}\label{eq:reconstruction error}
\fu(u,w) \triangleq E[ \|h_u(g(x;u);w) - x\|^2 ],
\end{equation}
is minimized (i.e., $\min_w \fu(u,w)$.)

For another example of utility, 
let $z$ be a target variable that is of interest to the subjects and analysts
such as medical diagnosis of users' data. 
An analyst can make a prediction $h_u(g(x);w)$ parameterized by $w \in \Wc$, 
which is assumed to be continuous in $w$ and continuously differentiable w.r.t.
the input $g(x)$.
The (dis)utility of the filtered output for a non-adversarial analyst can
also be measured by the expected risk
\begin{equation}\label{eq:utility risk}
\fu(u,w) \triangleq E[l_u(h_u(g(x;u);w), z)],
\end{equation}
where $l_u(\cdot)$ is a continuously differentiable loss function,
such that $\fu$ is continuously differentiable w.r.t. the filter
parameter $u$. 
%\subsection{Notes on min vs inf.}
To facilitate the analysis, the paper assumes that the constraint sets $\Wc$,
$\Vc$, and $\Uc$ are compact and convex subsets of Euclidean spaces such as
a ball with a large but finite radius.
Along with the assumption that the filter $g$ and the risks $\fp$ and $\fu$ 
are all continuous, min and max values are bounded and attainable. 
In addition, the solutions to min or max problems are assumed to be 
in the interior of $\Wc$, $\Vc$, and $\Uc$, enforced by adding appropriate
regularization (e.g, $\lambda \|w\|^2$) to the optimization problems
if necessary. 
For this reason, min or max problems that appear in the paper will be treated as
unconstrained and the notations $u\in\Uc$, $v\in\Vc$, and $w\in\Wc$
will be omitted.

Having defined the privacy measure and the utility measure, 
the goal of a filter designer is to find a filter that achieves the following two objectives.
The first objective is to {\it maximize privacy}
\begin{equation}\label{eq:privacy goal}
\max_u \min_v \fp(u,v) \;\;\;(\mathrm{or}\;\mathrm{equivalently,}\;\; \min_u \max_v -\fp(u,v))
\end{equation}
where $\min_v \fp(u,v)$ represents the risk of the worst (i.e., most capable)
adversary: the smaller the risk, the more accurately can she infer
the sensitive variable $y$.
As mentioned before, this problem alone has a trivial solution such as a 
constant filter that outputs zeros. 
The second objective is to {\it minimize disutility}
\begin{equation}\label{eq:utility goal}
\min_u \min_w \fu(u,w)\;\;\;(\mathrm{or}\;\mathrm{equivalently,}\;\;  \min_u -\max_w -\fu(u,w))
\end{equation}
where $\min_w \fu(u,w)$ represents the risk of the best analyst:
the smaller the risk, the more accurately can the analyst reconstruct 
original data $x$ or predict the variable of interest $z$. 
To achieve the two opposing goals, we can solve the joint problem of 
minimizing the weighted sum 
\begin{equation}\label{eq:joint goal 1}
\min_u \;\left[ \max_{v} -\fp(u,v) + \rho\;\min_{w} \fu(u,w)\right],
\end{equation}
or equivalently the weighted difference of max values
\begin{equation}\label{eq:joint goal 2}
\min_u \;\left[\max_v -\fp(u,v)-\rho\max_w -\fu(u,w)\right].
\end{equation}
The constant $\rho>0$ determines the relative importance of utility versus privacy.
For a small $\rho \ll 1$, the problem is close to a trivial privacy-only task,
and for a large $\rho \gg 1$, the problem is close to a utility-only task.
The solution to (\ref{eq:joint goal 1}) or (\ref{eq:joint goal 2})
will be referred to as {\bf minimax filter}\footnote{To be precise, the joint task (\ref{eq:joint goal 1}) is
a min-diff-max problem and the privacy-only task (\ref{eq:privacy goal})
is a minimax problem. However, both will be referred to as minimax
as (\ref{eq:joint goal 1}) can be rewritten as a standard minimax problem.}
and is by definition an optimal 
filter for utility-privacy tradeoff in terms of expected risks 
given the family of filters $\{g(\cdot\;;u)|u\in\Uc\}$, the family of 
private losses/classifiers $\{l_p(h_p(\cdot\;;v))|v\in\Vc\}$ and the family of utility
losses/classifiers $\{l_u(h_u(\cdot\;;w))|w\in\Wc\}$. 
Note that the choice of filter and loss/classifier families is very flexible,
with the assumption of differentiability only. In practice, almost-everywhere
differentiability suffices to use the algorithm in the paper. 
Fig.~\ref{fig:minimax filter} shows an example filter/classifier from the class
of multilayer neural networks.
As an aside, the joint problem may be formulated as minimization of
disutility with a hard constraint on privacy risk. 
When using interior-point methods, the procedure is similar to solving
(\ref{eq:joint goal 2}) iteratively with an increasing $\rho$, 
which is more demanding than minimizing the weighted sum only once as the paper 
proposes.

\begin{figure}[tb]
\centering
\includegraphics[width=0.6\linewidth]{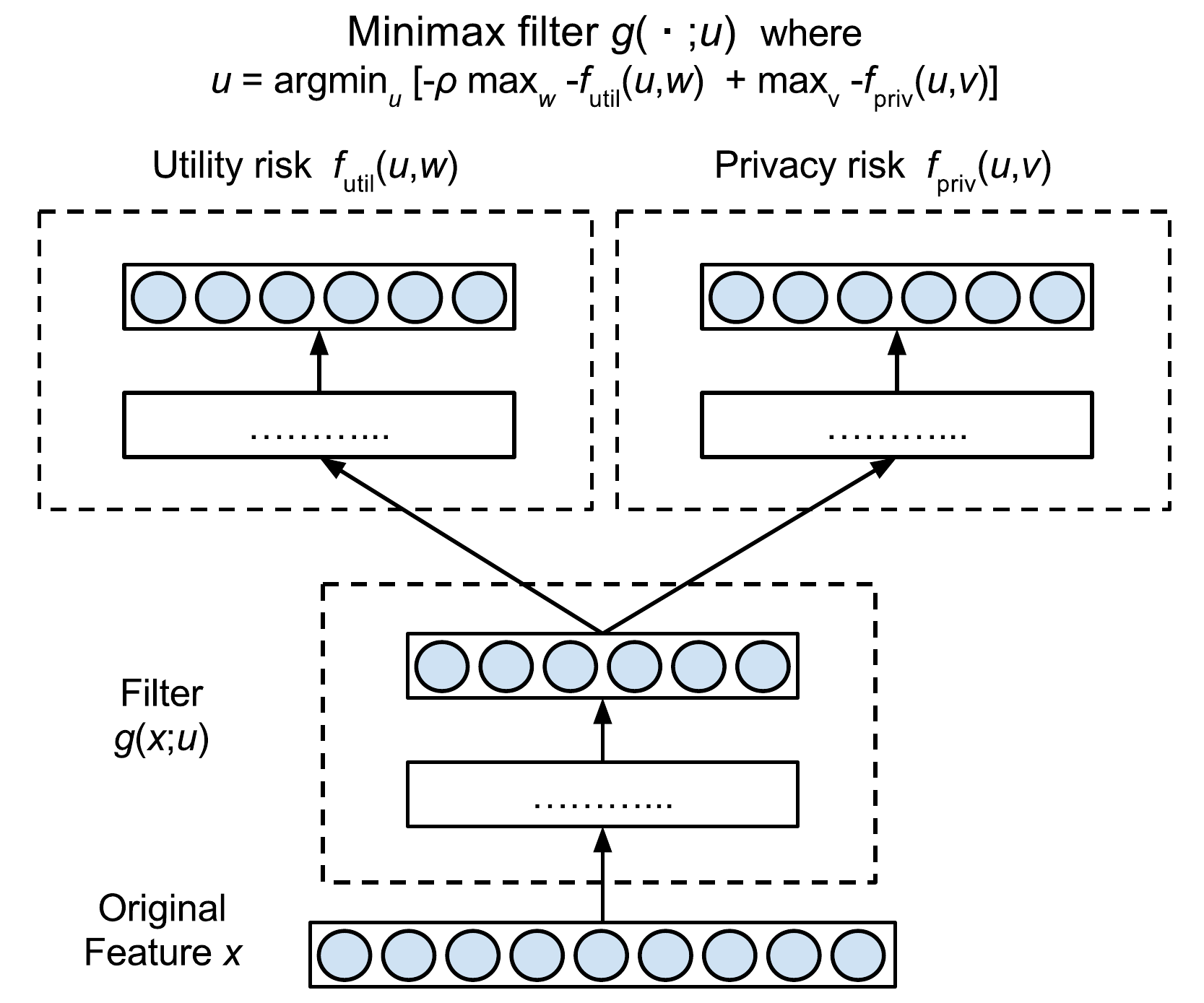}
\caption{Minimax filter with a filter/classifiers from the class
of multilayer neural networks.
}
\label{fig:minimax filter}
\end{figure}

\subsection{Notes on private and utility tasks}

The private variable $y$ can be any attribute which is considered sensitive
or identifying. 
For example, let $y$ be any number or string unique to a person in the
data set. Such identifiers are bijective with $\{1, ... ,S\}$, where
$S$ is the total number of subjects, so assume $y\in \{1,...,S\}$. 
The private task for an adversary is then to predict the subject number $y$
from the filtered data $g(x)$, whose inaccuracy is measured by the expected risk of
the identification task. 
That is, the less accurate the private classifier is, the more anonymous 
the filtered output is. 
The identity variable can also be group identifiers, e.g., $y$ is a
demographic grouping based on age, sex, ethnicity, etc.
Another example of private tasks is to single-out a particular subject among the rest,
in which case $y$ is binary: $y =-1$ means `not the target subject' and $y=1$ means `target subject'.
To summarize, anonymity of filtered data in this paper means 
resilience to inference attacks on any variable $y$ that we choose. 
This unifying approach is convenient since we need not determine
whether an attribute is an identifier,
a quasi-identifier or a sensitive attribute as in syntactic anonymization. 
Any information hidden in the continuous high-dimensional features
which are relevant to the private variable $y$---whatever it may be---will be
maximally filtered out by construction.
Similarly, the target variable $z$ of interest can be any variable that is 
not the same as the private variable $y$.
In the pathological case where they are the same ($z=y$), 
the objective (\ref{eq:joint goal 2}) becomes 
\begin{equation}
\min_u\; \left[(1-\rho) \max_v -f(u,v)\right]
\end{equation}
 which is either a trivial privacy-only problem when $0\leq \rho < 1$, or
a utility-only problem when $\rho > 1$.
In general, $z$ and $y$ will be correlated to a certain degree, and
the minimax filter will find the best compromise of utility and privacy risks.

Also, private and target tasks need not be classification tasks only.
Regression tasks can also be used as a target task,
as well as unsupervised tasks that do not require label $z$.
Unsupervised tasks are useful when the target task is unknown or non-specific.
For example, (\ref{eq:reconstruction error}) measures the expected least-squares error
between the original and the reconstructed features. 
%Note that exact reconstruction $h(g(x))=x$ is no possible when the filter is 
%non-invertible.

\subsection{Multiple tasks}
Extension of the idea in the previous section to multiple private and target tasks is straightforward.
Suppose there are $N_p$ private tasks $\fp^1(u,v_1),...,\fp^{N_p}(u,v_{Np})$
associated with private random variables $y^1,...,y^{Np}$.
Note that $\fp^i(u,v_i) = E[l_p^i(h^i(g(x;u);v_i),y^i)]$.
Similarly, suppose there are 
$N_u$ target tasks $\fu^1(u,w_1),...,\fu^{N_u}(u,w_{Nu})$
associated with target random variables $z_1,...,z_{Nu}$.
If $\kappa_1,...,\kappa_{Np}$ are the coefficients representing
relative importance of private tasks, and $\rho_1,...,\rho_{Nu}$ are 
the coefficients for utility tasks, then the final objective is to solve the following problem
\begin{equation}\label{eq:joint goal 3}
\min_u \;\left[ \sum_{i=1}^{N_p} \kappa_i \max_{v_i} -\fp^i(u,v_i) 
\;+\; \sum_{i=1}^{N_u} \rho_i \min_{w_i} \fu^i(u,w_i)\;\right].
\end{equation}
Since this can be rewritten as
\begin{equation}
\min_u \;\left[ \max_{v_1,...,v_{N_p}} -\sum_{i=1}^{N_p} \kappa_i\fp^i(u,v_i) 
\;+\; \min_{w_1,...,w_{N_u}} \sum_{i=1}^{N_u} \rho_i \fu^i(u,w_i)\;\right],
\end{equation}
this multiple task problem is nearly identical to the original single task problem (\ref{eq:joint goal 1}),
with the new utility and privacy tasks defined as
\begin{eqnarray}
\hat{f}_{\mathrm{priv}}(u, v\mathrm{=}(v_1,...,v_{Np}))&\triangleq&\sum_i \kappa_i \fp^i(u,v_i)
,\;\;\mathrm{and}\\
\hat{f}_{\mathrm{util}}(u, w\mathrm{=}(w_1,...,w_{Nu}))&\triangleq&\sum_i \rho_i \fu^i(u,w_i),
\end{eqnarray}
with $\rho=1$. 
Using this, it is straightforward to extend the analysis 
and algorithms developed for single tasks to those for multiple tasks.

\subsection{Generalization performance of minimax filter}\label{sec:analysis}

The proposed privacy mechanism is a learning-based approach. 
An optimal filter is one that solves the expected risk optimization 
(\ref{eq:joint goal 1}). 
However, in reality, an optimal filter has to be estimated from finite 
training samples, and we need a guarantee on the performance
of the learned filter on unseen test samples.
This section derives generalization bounds for empirical minimax filter
similar to the derivation of bounds for empirical risk minimizers (ERM).

The joint problem for expected risks was
\begin{equation}
\min_u \left[\max_v -\fp(u,v) + \rho \min_w \fu(u,w)\right]\\
%&=& \min_u \left[\max_v -E[l_p(h_p(g(x;u);v),y)] + \rho \min_w 
%E[l_u(h_u(g(x;u);w),z)]\right]\\
= \min_u \left[\max_v -E[l_p(u,v)] + \rho \min_w 
E[l_u(u,w)]\right].
\end{equation}
A joint loss $l_J$ is introduced for convenience:
\begin{equation}\label{eq:joint loss}
l_J(u,v,w) \triangleq -l_p(u,v) + \rho\; l_u(u,w).
\end{equation}
Let $(\us,\vs,\ws)$ be a solution to the expected risk optimization problem: 
\begin{equation}
E_D[l_J(\us,\vs,\ws)] = 
\min_u \left[\max_v E_D[-l_p(u,v)] + \rho \min_w E_D[l_u(u,w)] \right],
\end{equation}
where $E_D[\cdot]$ is the expected value w.r.t. the unknown data distribution 
$P(x,y)$.
Similarly, let $(\uh,\vh,\wh)$ be a solution to the empirical risk minimax problem: 
\begin{equation}
E_S[l_J(\uh,\vh,\wh)] = 
\min_u \left[\max_v E_S[-l_p(u,v)] + \rho \min_w E_S[l_u(u,w)] \right],
\end{equation}
where the empirical mean $E_S[\cdot]$ for $S=\{(x_1,y_1),\cdots,(x_N,y_N)\}$ is
\begin{equation}
E_S[l(x,y)] \triangleq \frac{1}{N}\sum_{(x,y)\in S} l(x,y).
\end{equation}
%If $(\uh,\vh,\wh)$ and $(\us,\vs,\ws)$ are the optimal
%solutions for empirical and expected $f$ above, 
The goal in this analysis is to show that the expected and the empirical optimizers
perform equally well in expectation/probability given enough training samples:
\begin{equation}
E_D[l_J(\uh,\vh,\wh)] \simeq E_D[l_J(\us,\vs,\ws)],\;\;\mathrm{as}\;\;N\to\infty.
\end{equation}
The main result is the Theorem~\ref{thm:generalization bounds}
which is proved in the remainder of this section. 
Let's define optimal parameters $v(u)$ and $w(u)$ given $u$ as
\begin{eqnarray}
\vs(u) &\triangleq& \arg\max_v E_D[-l_p(u,v)],\;\;\;\;\;\vh(u)\;\;\triangleq\;\;\arg\max_v E_S[-l_p(u,v)], \\
\ws(u) &\triangleq& \arg\min_w E_D[l_u(u,w)],\;\;\;\;\;\;\;\wh(u)\;\;\triangleq\;\;\arg\min_w E_S[l_u(u,w)].
\end{eqnarray}
One can then write
\begin{eqnarray}
E_D[l_J(\us,\vs,\ws)] &=& \min_u \left[ E_D[-l_p(u,\vs(u))] +\rho E_D[l_u(u,\ws(u))\right]]\\
&=&\min_u E_D[l_J(u,\vs(u),\ws(u))],
\end{eqnarray}
and similarly
\begin{equation}
E_S[l_J(\uh,\vh,\wh)] = \min_u \left[ E_S[-l_p(u,\vh(u))] +\rho E_S[l_u(u,\wh(u))]\right]=\min_u E_S[l_J(u,\vh(u),\wh(u))].
\end{equation}
From these definitions we have for all $u$, 
\begin{eqnarray}
E_D[l_J(\us,\vs,\ws)] & \leq & E_D[l_J(u,\vs(u),\ws(u))], \label{eq:ineq1}\\
E_S[l_J(\uh,\vh,\wh)] & \leq & E_S[l_J(u,\vh(u),\wh(u))].\label{eq:ineq2}
\end{eqnarray}
Also from definition, for all $(u,v,w)$,
\begin{eqnarray}
E_D[l_J(u,v,\ws(u))] &\leq& E_D[l_J(u,v,w)] \;\leq\; E_D[l_J(u,\vs(u),w)],\label{eq:ineq3}\\
E_S[l_J(u,v,\wh(u))] &\leq& E_S[l_J(u,v,w)] \;\leq\; E_S[l_J(u,\vh(u),w)].\label{eq:ineq4}
\end{eqnarray}
These observations imply the following theorem.
\begin{theorem}\label{thm:representativeness}
The risk difference of expected and empirical optimizers is at most
twice the largest difference of expected and empirical risks of 
any set of parameters: 
\begin{equation}\label{eq:bound}
| E_D[l_J(\uh,\vh,\wh)]  - E_D[l_J(\us,\vs,\ws)]| \leq 
2 \sup_{u,v,w} \left|E_D[l_J(u,v,w)] - E_S[l_J(u,v,w)]\right|.
\end{equation}
\end{theorem}
\begin{proof}
The expected risk of empirical risk optimizers $(\uh,\vh,\wh)$ is 
upper-bounded by the risk of expected risk optimizers $(\us,\vs,\ws)$ as follows:
\begin{eqnarray*}
&& E_D[l_J(\uh,\vh,\wh)] - E_D[l_J(\us,\vs,\ws)] \\
&=& E_D[l_J(\uh,\vh,\wh)] -E_S[l_J(\uh,\vh,\wh)] - \left(E_D[l_J(\us,\vs,\ws)] -E_S[l_J(\uh,\vh,\wh)]\right)\\
&\leq& E_D[l_J(\uh,\vh,\wh)] -E_S[l_J(\uh,\vh,\wh)] 
- \left(E_D[l_J(\us,\vs,\ws)] -E_S[l_J(\us,\vh(\us),\wh(\us))]\right) \;\;(\mathrm{from}\;(\ref{eq:ineq2}))\\
&\leq& E_D[l_J(\uh,\vh,\wh)] -E_S[l_J(\uh,\vh,\wh)]
- \left(E_D[l_J(\us,\vh(\us),\ws)] -E_S[l_J(\us,\vh(\us),\wh(\us))]\right)\;\; (\mathrm{from}\;(\ref{eq:ineq3}))\\
&\leq& E_D[l_J(\uh,\vh,\wh)] -E_S[l_J(\uh,\vh,\wh)]
- \left(E_D[l_J(\us,\vh(\us),\ws)] -E_S[l_J(\us,\vh(\us),\ws)]\right)\;\; (\mathrm{from}\;(\ref{eq:ineq4}))\\
&\leq& 2\sup_{u,v,w} \left| E_D[l_J(u,v,w)] - E_S[l_J(u,v,w)] \right|.
\end{eqnarray*}
The difference can also be lower-bounded as follows:
\begin{eqnarray*}
&& E_D[l_J(\us,\vs,\ws)] - E_D[l_J(\uh,\vh,\wh)] \\
&=& E_D[l_J(\us,\vs,\ws)] - E_S[l_J(\uh,\vh,\wh)] - \left(E_D[l_J(\uh,\vh,\wh)]-E_S[l_J(\uh,\vh,\wh)]\right)\\
&\leq & E_D[l_J(\uh,\vs(\uh),\ws(\uh))] - E_S[l_J(\uh,\vh,\wh)] - \left(E_D[l_J(\uh,\vh,\wh)]-E_S[l_J(\uh,\vh,\wh)]\right) \;\;(\mathrm{from}\;(\ref{eq:ineq1}))\\
&\leq & E_D[l_J(\uh,\vs(\uh),\ws(\uh))] - E_S[l_J(\uh,\vs(\uh),\wh)] - \left(E_D[l_J(\uh,\vh,\wh)]-E_S[l_J(\uh,\vh,\wh)]\right) \;\;(\mathrm{from}\;(\ref{eq:ineq4}))\\
&\leq & E_D[l_J(\uh,\vs(\uh),\wh)] - E_S[l_J(\uh,\vs(\uh),\wh)] - \left(E_D[l_J(\uh,\vh,\wh)]-E_S[l_J(\uh,\vh,\wh)]\right) \;\;(\mathrm{from}\;(\ref{eq:ineq3}))\\
&\leq& 2\sup_{u,v,w} \left| E_D[l_J(u,v,w)] - E_S[l_J(u,v,w)] \right|.
\end{eqnarray*}
\end{proof}
To bound the RHS of (\ref{eq:bound}), one can use the Rademacher complexity theory 
(e.g., Lemma 26.2 of \cite{Shalev-Shwartz:2014}.)
\begin{lemma}\label{lemma:rademacher}
Let $F$ be a class of real-valued functions, and let
$S$ be a set of $N$ samples $S = \{(x_1,y_1),...,(x_N,y_N)\}$.
Then, 
\begin{equation}
E_{S \sim D^N} \left[ \sup_{f \in F} |E_D[f] - E_S[f]| \right] 
\leq 2 E_{S \sim D^N}[\mathfrak{R}(F \circ S)],
\end{equation}
where $\mathfrak{R}(F \circ S)$ is the empirical Rademacher complexity 
\begin{equation}
{\mathfrak{R}}(f \circ S) \triangleq \frac{1}{N} E_{\sigma \sim \{-1,+1\}^N}
\left[ \sup_{f \in F} \sum_{i=1}^N \sigma_i f(x_i,y_i)\right]
\end{equation}
for the class of real-valued functions 
$\{(x,y) \mapsto f(x,y)\;:\;\forall f \in F\}$.
\end{lemma}
Consider the class of real-valued functions defined from the joint loss (\ref{eq:joint loss}):
\begin{eqnarray}
l_J \circ H_J \circ S &\triangleq& \{(x,y,z) \mapsto 
l_J(x,y,z;u,v,w)\;:\;u\in \Uc, v \in \Vc, w \in \Wc\}\\
&=& \{(x,y,z) \mapsto -l_p(h_v(g_u(x)),y) + \rho\;l_u(h_w(g_u(x)),z)\;:\;
u \in \Uc, v \in \Vc, w \in \Wc\}. \nonumber
\end{eqnarray}
Let ${\mathfrak{R}}(l_J \circ H_J \circ S)$
denote the empirical Rademacher complexity of the joint loss class.
Furthermore, the Rademacher complexity of sum of functions can be upper-bounded 
by the sum of complexities:
\begin{lemma}\label{lemma:rademacher sum}
The empirical Rademacher complexity of the joint privacy-utility loss is 
upper-bounded as
\begin{equation}
\mathfrak{R}(l_J \circ H_J \circ S)
\leq \mathfrak{R}(l_p\circ H_p \circ G \circ S) + \rho\; \mathfrak{R}(l_u\circ H_u \circ G \circ S),\;\;\;(\rho>0)
\end{equation}
where 
\begin{eqnarray}
l_p \circ H_p \circ G \circ S 
&\triangleq& \{(x,y,z) \mapsto l_p(h_p(g_u(x)),y)\;:\;u \in \Uc, v \in \Vc\}, \\
l_u \circ H_u \circ G \circ S 
&\triangleq& \{(x,y,z) \mapsto l_u(h_u(g_u(x)),z)\;:\;u \in \Uc, w \in \Wc\}.
\end{eqnarray}
\end{lemma}

\begin{proof}
\begin{eqnarray}
\mathfrak{R}(l_J \circ H_J \circ S)&=&\frac{1}{N} E_{\sigma}\left[ \sup_{u,v,w} \left| \sum_{i=1}^N \sigma_i l_J(x_i,y_i,z_i;u,v,w) \right| \right]\\
&=& \frac{1}{N} E_{\sigma}\left[ \sup_{u,v,w} \left| \sum_{i=1}^N \sigma_i (-l_p(x_i,y_i;u,v) + \rho\;l_u(x_i,z_i;u,w)) \right| \right]\\
%&\leq& \frac{2}{m} E_{\sigma}\left[ \sup_{u,v,w} \left| \sum_{i=1}^m -\sigma_i  l_p(z_i;u,v) + \sum_{i=1}^m \rho\; \sigma  l_u(z_i;u',w) \right| \right]\\
&\leq& \frac{1}{N} E_{\sigma}\left[ \sup_{u,v} \left| \sum_{i=1}^N \sigma_i l_p(x_i,y_i;u,v)\right| + \rho\;\sup_{u,w} \left| \sum_{i=1}^N \sigma_i l_u(x_i,z_i;u,w) \right| \right]\\
&=& \mathfrak{R}(l_p\circ H_p \circ G \circ S) + \rho\; \mathfrak{R}(l_u\circ H_u \circ G \circ S).
\end{eqnarray}
\end{proof}
\if0
Lastly, the McDiarmid's inequality is used: 
\begin{lemma}\label{lemma:mcdiarmid}
Let $V$ be some set and let $f : V^m \to \mathbb{R}$
be a function of $m$ variables such that for some $c > 0$, for all $i \in \{1,..,m\}$
and for all $x_1, ... , x_m,\;\;x_i' \in V$  we have
\begin{equation}
|f(x_1 , ... , x_m)-f(x_1,... , x_{i-1}, x_i' , x_{i+1},..., x_m )| \leq c.
\end{equation}
Let $X_1,..., X_m$ be $m$ independent random variables taking values in $V$.
Then, with probability of at least $1-\delta$ we have
\begin{equation}
|f(X_1 , ... , X_m)-E[f(X_1,...,X_m)]| \leq c\sqrt{\ln(2/\delta) m/2}.
\end{equation}
\end{lemma}
\fi
From Theorem~\ref{thm:representativeness} and Lemmas~\ref{lemma:rademacher} and 
~\ref{lemma:rademacher sum},
we get the following generalization bounds in terms of the Rademacher complexity.
\begin{theorem}\label{thm:generalization bounds}
\begin{eqnarray}
&&E_{S \sim D^m}\left[|E_D[l_J(\us,\vs,\ws)] - E_D[l_J(\uh,\vh,\wh)]|\right]\nonumber\\
&&\;\;\;\;\;\;\;\;\;\;\;\;\leq \;\;4 
E_{S \sim D^m}\left[
\mathfrak{R}(l_p\circ H_p \circ G \circ S) + \rho\; \mathfrak{R}(l_u\circ H_u \circ G \circ S)\right].
\end{eqnarray}
\if0
Furthermore, suppose the loss functions are bounded ($|l_p| \leq c$ and
$|l_u|\leq c$.) Then, with probability of at least 1-$\delta$, we have
\begin{eqnarray}
|E_D[l_J(\us,\vs,\ws)] - E_D[l_J(\uh,\vh,\wh)]|
&\leq& 4 E_{S \sim D^m}\left[
\mathfrak{R}(l_p\circ H_p \circ G \circ S) + \rho\; \mathfrak{R}(l_u\circ H_u \circ G \circ S)\right]\nonumber\\
&& +\; c \sqrt{\frac{2 \log(2/\delta)}{m}}.
\end{eqnarray}
\fi
\end{theorem}
A probabilistic bound instead of expected value can also be obtained by
applying McDiarmid's inequality, which is omitted. 

%\begin{proof}
%It is immediate from
%Theorem~\ref{thm:representativeness} and Lemmas~\ref{lemma:rademacher} and
%\ref{lemma:rademacher sum}. 
%The second part follows from the McDiarmid's inequality (Lemma~\ref{lemma:mcdiarmid})
%with the constant of $2c/m$.
%\end{proof}
The Rademacher complexity of privacy and utility losses depends on our
choice of loss functions, hypothesis classes, and filter classes.
For the simple case of linear filters and
linear classifiers, one can compute the complexity using the following lemmas
(26.9 and 26.10 from \cite{Shalev-Shwartz:2014}):
\begin{lemma}
Suppose $\phi:\mathbb{R} \to \mathbb{R}$ is $\alpha$-Lipschitz,
i.e., $|\phi(a)-\phi(b)| \leq \alpha |a - b|,\;\forall a,b \in \mathbb{R}$.
Then, 
\begin{equation}
\mathfrak{R}(\phi\circ F) = \alpha \mathfrak{R}(F).
\end{equation}
\end{lemma} 
\begin{lemma}
For the class of linear classifiers $H = \{x \mapsto w^Tx\;:\;\|w\|_2 \leq 1\}$, 
\begin{equation}
\mathfrak{R}(H \circ S) \leq \frac{1}{\sqrt{N}} \sup_{x\in S} \|x\|_2 .
\end{equation}.
\end{lemma}
\if0
\begin{proof}
Suppose $H = \{ z \mapsto w^Tz\;:\;\|w\|_2 = 1 \}$ is a family of linear functions
with unit $2$-norm.
Then, 
\begin{eqnarray}
\mathfrak{R}(H\circ G \circ S) &=& 
\frac{1}{m} E[\sup_{h\in H, g \in G} \sum_i \sigma_i w^T g(x_i)] 
=\frac{1}{m} E[\sup_{w, g} w^T \sum_i \sigma_i g(x_i) ]\\
&\leq&\frac{1}{m} E[\sup_{g} \|\sum_i \sigma_i g(x_i)\|_2].
\end{eqnarray}
Note that
\[
E[\|\sum_i \sigma_i g(x_i)\|_2]
\leq \sqrt{E[\|\sum_i \sigma_i g(x_i)\|_2^2]}.
\]
\end{proof}
\fi
From these lemmas and Theorem~\ref{thm:generalization bounds}, we
have a corollary for a simple case of linear filters and classifiers. 
\begin{corollary}
Let the loss functions $l_u$ and $l_p$ be $\alpha$-Lipschitz (e.g., 
$\alpha=1$ for logistic regression.)
Suppose $U$ is a $d \times D$ real matrix with a bounded norm
$\|U\|_2 \leq 1$, and $w$ and $v$ are vectors with bounded norms 
($\|w\|_2 \leq 1$ and $\|v\|_2 \leq 1$). 
If the feature domain $\mathcal{X}$ is also bounded with a radius 
$r = \max_{x \in \mathcal{X}} \|x\|_2$,  then we have
\begin{equation}
|E_D[l_J(\us,\vs,\ws)] - E_D[l_J(\uh,\vh,\wh)]|
\leq \frac{4 (1+\rho)\;\alpha\; r}{\sqrt{N}}.
\end{equation}
\end{corollary}
Alternatively, one can use the VC dimension to specify the bound. 

In any case, the generalization bounds in this section justify the claim 
that minimax filter can preserve utility-privacy of unseen test data 
in expectation or probability.

\if0
If in addition, $G = \{ x \mapsto Ax\;:\;\|A\|_2 = \rho\}$ is a family
of linear dimensionality reduction with bounded 2-norm $\rho$, then
\begin{eqnarray}
\mathfrak{R}(H\circ G \circ S) &\leq& 
\frac{1}{m} E[\sup_{A} \|\sum_i \sigma_i A x_i\|_2]
\leq \frac{1}{m} \sqrt{E[\sup_{A} \|\sum_i \sigma_i A x_i\|^2_2]}\\
&=& \frac{1}{m} \sqrt{E[\sup_{A} \sum_{i\ne j} \sigma_i \sigma_j x_i^TA^TA x_i
+ \sum_i \sigma_i^2 x_i^TA^TAx_i]} \\
&\leq & \frac{1}{m} \sqrt{E[ \sum_{i\ne j} \sigma_i \sigma_j \rho^2 x_i^T x_i
+ \sum_i \sigma_i^2 \rho^2 x_i^Tx_i]} 
=  \frac{1}{m} \sqrt{\sum_i \rho^2 \|x_i\|^2_2} \\
&\leq & \frac{\rho}{m} \sqrt{m \max_i \|x_i\|^2_2}
= \frac{\rho}{\sqrt{m}} \max_i \|x_i\|_2.
\end{eqnarray}
\fi

%%%%%%%%%%%%%%%%%%%%%%%%%%%%%%%%%%%%%%%%%%%%%%%%%%%%%%%%%%%%%%%%%%%%%%%%%%%%%%%%%%%%%%
\section{Minimax Optimization}\label{sec:optimization}
%%%%%%%%%%%%%%%%%%%%%%%%%%%%%%%%%%%%%%%%%%%%%%%%%%%%%%%%%%%%%%%%%%%%%%%%%%%%%%%%%%%%%%

This section presents theoretical and numerical solutions of the joint problem
(\ref{eq:joint goal 2}),
which is a variant of unconstrained continuous minimax problems. 
(See \cite{Rustem:2009} for a review.)
The problem (\ref{eq:joint goal 2}) can be written in an equivalent form
%\begin{equation}\label{eq:joint goal 1}
%\min_u \;[ \max_{v} -\fp(u,v) + \rho\;\min_{w} \fu(u,w)\;],
%\end{equation}
%\begin{equation}\label{eq:joint goal 2}
%\min_u [\max_v -\fp(u,v)-\rho\max_w -\fu(u,w)],
%\end{equation}
\begin{eqnarray}
\min_u \Phi(u)&=&\min_u[\Phi_{\mathrm{priv}}(u) - \rho\;\Phi_{\mathrm{util}}(u)]\label{eq:Phi}\\
&=&\min_u[\max_{v} -\fp(u,v) - \rho\; \max_{w} -\fu(u,w)]
\end{eqnarray}
%Phi(u) &=&  \Phi_{\mathrm{priv}}(u) - \rho\;\Phi_{\mathrm{util}}(u)\\
%\Phi_{\mathrm{priv}}(u) &=& \max_{v} \fp(u,v)\\
%\Phi_{\mathrm{util}}(u) &=& \max_{w} (-\fu(u,w)).
%\end{eqnarray}
The optimization above is a min-diff-max problem and can be considered as
simultaneously solving two subproblems
$\min_u [\max_v -\fp(u,v)]$ and $\min_u [-\max_w -\fu(u,w)]$,
but is evidently not the same as summing individual solutions
\begin{equation}
\min_u \Phi(u) \ne \min_u [\max_v -\fp(u,v)]\;+\; \min_u [-\rho\max_w -\fu(u,w)].
\end{equation}
Since the second subproblem $\min_u [-\max_w -\fu(u,w)]=\min_{u,w} \fu(u,w)$ 
is a standard minimization problem, 
let's focus only on the first subproblem $\min_u [\max_v -\fp(u,v)]$
which is a continuous minimax problem.
Continuous minimax problems are in general more challenging to solve
than standard minimization problems, as the inner optimization
 $\Phi_{\mathrm{priv}}(u)=\max_v -\fp$ 
does not usually have a closed-form solution;
when it does, the whole problem can be treated as a standard minimization problem.
Furthermore, there can be more than one solution to  $\Phi_{\mathrm{priv}}(u)=\max_v -\fp$.
To better understand minimax problems, we look at several examples starting
from a simple case
where $\Phi_{\mathrm{priv}}$ and $\Phi_{\mathrm{util}}$ have closed-form solutions.

\subsection{Simple case: eigenvalue problem}

%For illustration of the minimax formulation, a simple example with a
%closed-form solution is demonstrated in the following. 
Consider finding a minimax filter for the following problem.
The filter class is a linear dimensionality reduction ($g(x;u)= U^Tx$)
parameterized by the matrix $U\in \mathbb{R}^{D\times d}$,
and the private and target tasks are least-squares regressions parameterized
by the matrices $V$ and $W$:
\begin{eqnarray}
\fp(U,V) &=& \frac{1}{N} \sum_i \|V^T U^Tx_i - y_i\|^2,\;\;\mathrm{and}\\
%= \frac{1}{N}\|V^TU^TX - Y\|_F^2,\;\mathrm{and}\\
\fu(U,W) &=& \frac{1}{N} \sum_i \|W^T U^Tx_i - z_i\|^2.
%= \frac{1}{N}\|W^TU^TX - Z\|_F^2.
\end{eqnarray}
In this case, $\Phi_{\mathrm{priv}}(U)=\max_{V} -f_p(U,V)$ and
$\Phi_{\mathrm{util}}(U)=\max_{W} -f_u(U,W)$ are both concave problems with
closed-form solutions
\begin{eqnarray}
\hat{V}&=&\arg\min_V \fp = (U^TC_{xx}U)^{-1} U^TC_{xy}\;\;\mathrm{and}\\
\hat{W}&=&\arg\min_W \fu = (U^TC_{xx}U)^{-1} U^TC_{xz},
\end{eqnarray}
where 
\begin{equation}
C_{xy} = \frac{1}{N} \sum_i x_i y_i^T,\;\;C_{xz} = \frac{1}{N} \sum_i x_i z_i^T,\;\;\mathrm{and}\;\;C_{xx} = \frac{1}{N} \sum_i x_i x_i^T.
\end{equation}
The corresponding min values are
\begin{eqnarray}
\Phi_{\mathrm{priv}}(U) &=& -\fp(U,\hat{V}) =  \mathrm{Tr}\left[ (U^TC_{xx}U)^{-1}U^TC_{xy}C_{xy}^TU\right]  + \mathrm{const},\;\mathrm{and}\\
\Phi_{\mathrm{util}}(U) &=& -\fu(U,\hat{W}) =  \mathrm{Tr}\left[ (U^TC_{xx}U)^{-1}U^TC_{xz}C_{xz}^TU\right]  + \mathrm{const}.
\end{eqnarray}
The outer minimization over $u$ is then
\begin{eqnarray}%\label{eq:simple case}
\min_U \Phi(U)&=&\min_U\left[\Phi_\mathrm{priv}(U)-\rho \Phi_{\mathrm{util}}(U)\right]\\
&=& \min_U\left[-\fp(U,\hat{V}) + \rho \fu(U,\hat{W})\right]
= \min_U \mathrm{Tr} \left[(U^TC_{xx}U)^{-1}\;U^TC_{xyz} U\right],\label{eq:simple case}
\end{eqnarray}
where 
\begin{equation}
C_{xyz} = C_{xy}C_{xy}^T - \rho C_{xz}C_{xz}^T.
\end{equation}
The problem (\ref{eq:simple case}) can be reformulated as a generalized eigenvalue
problem. 
Let $Q = C_{xx}^{1/2} U$ be a $D \times d$ full-rank matrix. The problem can be
rewritten as
\begin{equation}
\min_{U}\mathrm{Tr} \left[(U^TC_{xx}U)^{-1}\;U^TC_{xyz} U\right]\\
= \min_Q \mathrm{Tr} \left[ (Q^TQ)^{-1} Q^T C_{xx}^{-1/2} C_{xyz} C_{xx}^{-1/2} Q \right].
\end{equation}
Furthermore, note that min value (\ref{eq:simple case}) is invariant to the 
right multiplication of $U$ by any $d \times d$ nonsingular matrix $R$.
So chose $R$ so that $Q^TQ = R^T U^T C_{xx} UR = I_d$ without loss of generality.
Let $A = (C_{xx}^{-1/2})^T C_{xyz} C_{xx}^{-1/2}$, and the minimax problem
becomes the following eigenvalue problem:
\begin{equation}\label{eq:eigenvalue problem}
\min_U \Phi(U)= \min_{\{Q\;|\;Q^TQ=I_d\}}\; \mathrm{Tr}\; Q^T A Q,
%[\max_V -\fp(U,V) - \rho \max_W -\fu(U,W)] =
\end{equation}
which is the sum of the $d$ smallest eigenvalues of $A$ which may not be
positive semidefinite.
Note that this special case problem is quite similar to the objective of \cite{Enev:2012}:
\[
\max_u \;\;[-\lambda u^TC_{xy}^TC_{xy}u + u^TC_{xz}^TC_{xz}u ],\;\;\mathrm{s.\;t.}\;\;u^Tu = 1.
\]

The paper also proposes a variant of the eigenvalue problem (\ref{eq:eigenvalue problem}), 
called {\bf Privacy LDS} which is an analogue of linear discriminant analysis (LDS)
for privacy-utility optimization problem. 
Define the symmetric positive semidefinite matrix $C_u$ as
\begin{equation}
C_u = \sum_{k=1}^{K} N_k (\mu_k - \mu)(\mu_k - \mu)^T,
\end{equation}
where 
\begin{equation}
z \in \{1,...,K\},\;\; 
N_k = \sum_{i=1}^N I[z_i = k],\;\;
\mu_k = \frac{1}{N_k} \sum_{i=1}^N x_i I[z_i = k],\;\;\mathrm{and}\;\;
\mu = \frac{1}{N} \sum_{i=1}^N x_i,
\end{equation}
Define $C_p$ similarly with
\begin{equation}
y \in \{1,...,K'\},\;\;N_k' = \sum_{i=1}^{N'} I[y_i = k],\;\;\mathrm{and}\;\;\mu_k' =\frac{1}{N'_k} \sum_{i=1}^{N'} x_i I[y_i = k].
\end{equation}
The proposed Privacy LDS is a linear filter $g(x;U) = U^Tx$, 
where $U=[u_1,...,u_d]$ is a matrix of top eigenvectors $u_i$'s from the
following generalized eigenvalue problem: 
\begin{equation}\label{eq:privacy lds}
\max_{\|u\|=1}\;\;\frac{u^T(C_u + \lambda I)u}{u^T(C_p + \lambda I)u}.
\end{equation}
This paper uses Privacy LDS as a heuristic to find the initial linear filter
before fine-tuning the parameter $u$ using a general optimization method 
presented in the following sections. 
Note that this initialization is applicable only to linear filters. 
%(For nonlinear filters such as multilayer perceptron, we find initial parameters
%by random values or by training sparse autoencoder which does not require target
%labels.)

\subsection{Saddle-point problem}

Continuous minimax problems cannot in general be solved in closed form and
require numerical solvers. 
% saddle
There is a subclass of continuous minimax problems which are easier to
solve than others. Saddle-point problems are minimax problems 
for which $f(u,v)$ is convex in $u$ and concave in $v$, such as the following
``saddle'' problem
\begin{equation}
\min_u \max_v f(u,v) = \min_u \max_v \;[ u^2 - v^2].
\end{equation}
Analogous to convex problems, 
$f(u,v)$ has a global optimum $(\us,\vs)$ which satisfies
\begin{equation}
f(\us,v) \leq f(\us,\vs) \leq f(u,\vs).
\end{equation}
The convergence rate of a simple subgradient-descent method for saddle-point
problems was previously analyzed by \cite{Nedic:2009}.
Unfortunately, the minimax problem $\min_u \max_v -\fp$ considered in this paper
is not a saddle-point problem even for a relatively simple case.
Suppose one chooses linear filters, convex differentiable losses 
(e.g., least-squares, logistic, or exponential losses) and linear classifiers
for the problem.
Then 
\begin{equation}
-\fp(u,v)=-E[l(g(u);v)]=-E[l(yv^TU^Tx)]
\end{equation}
is the negative expected value of the composition of a convex $l(\cdot)$
and a linear $U^Tx$, which is concave in $U$ and is also concave in $v$, 
which cannot be a saddle-point problem.
\if0
Then $-\fp(u,v)$ is concave in $v$ but not necessarily
convex in $u$, as shown in the following.
With a $d$-dimensional filter output  $g(u) = [g_1(u),...,g_d(u)]$,
the Hessian matrix of the composite function $-\fp=-E[l(g(u);v)]$ w.r.t. $u$ is 
\begin{equation}
H_u(-\fp) = J_u^T(g) H_g(-E[l]) J_u(g) - \sum_k \frac{dE[l]}{dg_k} H_u(g_k),
\end{equation}
where $H(\cdot)$ and $J(\cdot)$ are Hessian and Jacobian matrices.
Since the Hessian $H_u(-\fp)$ is the sum of 1) a quadratic form with 
negative semidefinite $H_g(-E[l])$ by convexity assumption and 2) 
an arbitrary linear combination of $H_u(g_k)$'s, 
it may or may not be positive/negative definite. 
For example, with linear filters, $H_u(g_k)=0$ by linearity and therefore 
$H_u(f) = -J_u^T(g) H_g(E[l]) J_u(g)$ is concave in $u$ instead of convex,
which makes it not a saddle-point problem. 
\fi
%However, the presence of local minima does not prevent algorithms from 
%finding a good solution in practice. In the experiments with real data sets,
%local minima did not pose a noticeable problem in achieving promising results.

\subsection{General problem}%\label{sec:alternating}

A general numerical solution to the optimization (\ref{eq:Phi}) is 
described in this section.
Let $f(u,v)$ be a real-valued function
$f: \Uc \times \mathcal{V} \to \mathbb{R}$,
where $\Uc$ and $\Vc$ are compact subsets of the Euclidean space.
Suppose $f$ is jointly continuous and has a continuous partial derivative
$\nabla_u f$ w.r.t. the first variable $u$.
The maximum over $v$
\begin{equation}%\label{eq:Phi}
\Phi(u) = \max_{v \in \mathcal{V}} f(u,v)
\end{equation}
has a property that $\Phi(u)$ is in general not differentiable in $u$ 
even if $f(u,v)$ is \citep{Danskin:1967}.
Suppose $V(u)$ is the set of maximizers of $f$ given $u$:
\begin{equation}
V(u) = \{ \hat{v}\in\mathcal{V} \;|\; f(u,\hat{v}) = \max_{v\in \mathcal{V}} f(u,v)\}.
\end{equation}
Danskin proved that the directional derivative $D_y \Phi(u)$ in any
direction $y\in \mathbb{R}^d$ can be written as the maximum directional
derivatives of $f(u,v)$ over all $\hat{v} \in V(u)$:
\begin{equation}
D_y \Phi(u) = \max_{\hat{v} \in V(u)} {D}_y f(u,\hat{v}),
\end{equation}
where $D_y f(u,v)$ is the directional derivative of $f$ w.r.t. $u$.
Furthermore, in the case where $V(u)$ is a singleton $\{\hat{v}(u)\}$
for each $u$, we have 
\begin{equation}
D_y \Phi(u) = D_y f(u,\hat{v}(u)).
\end{equation}

% Kiwiel
There are several classic minimax optimization algorithms using this property.
%which require only weak assumptions on $f$. %, thereby allowing broad families of losses and filters to be used for minimax. 
Suppose $f(u,v)$ is also continuously differentiable w.r.t. $v$, 
and $\nabla_u f$ is continuously differentiable w.r.t. $v$. 
A first-order method for minimax problems was proposed by \cite{Panin:1981}
and was later refined by \cite{Kiwiel:1987}.
The latter uses a linear approximation of $f$ at a fixed $\bar{u}$ along the direction $q$
\begin{equation}\label{eq:approximated f}
f^l (q,v) = f(\bar{u},v) + \langle \nabla_u f(\bar{u},v),\; q \rangle,
\end{equation}
and uses it to compute the approximate max value
\begin{equation}\label{eq:approximated phi}
\Phi^l(q) = \max_{v} f^l(q,v).
\end{equation}
Using this approximation, a line search can be performed along the descent direction $q$
that minimizes the max function $\Phi(\bar{u}+\alpha q)$.
In particular, with additional assumptions of Lipschitz continuity of $\nabla_u f$ and
compactness of $\Uc$ and $\Vc$, 
Kiwiel's algorithm monotonically decreases $f$ for each iteration and
converges to a stationary point $u^\ast$, i.e., a point $u$ for which 
$\max_v \langle\nabla_u f(u^\ast,v),\;q\rangle \geq 0$ for all directions $q$.  
Previously, \cite{Hamm:2015a} used Kiwiel's algorithm 
to solve the optimization problems (\ref{eq:Phi}).
However, one disadvantage of the method was its slow speed in practice, 
due to the auxiliary routine of finding the descent direction $q$
at each iteration, described in the supplementary material of \cite{Hamm:2015a}.  

%$\min_u  \Phi(u)$, where
%\[
%\Phi(u) = -\rho \max_w -\fu(u,w) + \max_v -\fp(u,v)
%    = \rho \min_w \fu(u,w) + \max_v -\fp(u,v).
%\]

%\subsection{Alternating optimization}\label{sec:alternating}
Instead, this paper proposes a simple alternating algorithm (Alg.~\ref{alg:alternating})
for solving min-diff-max problem based directly on Danskin's theorem.
The algorithm only assumes $\fp(u,v)$ and $\fu(u,w)$ to be jointly continuous and 
have continuous partial derivatives $\nabla_u \fp$ and $\nabla_u \fu$.
Additionally, if $\fp(u,v)$ and $\fu(u,w)$ are convex in $v$ and $w$ respectively,
then the global minima
\begin{equation}
v_t = \arg\min_v \fp(u_t,v)\;\;\mathrm{and}\;\;w_t = \arg\min_w \fu(u_t,w)
\end{equation}
can be found easily, either approximately or accurately. 
Furthermore, if $\fu$ and $\fp$ are strongly convex (e.g., due to regularization), 
the solutions are unique.
Consequently, the descent direction $q_t$ in Alg.~\ref{alg:alternating} is truly
the (negative) gradient of $\Phi(u)$ (\ref{eq:Phi})
as desired:
\begin{equation}
q_t = \nabla_u \fp(u,v_t)-\rho\nabla_u \fu(u,w_t)
= -\nabla_u \Phi_{\mathrm{priv}}(u) + \rho\;\nabla_u \Phi_{\mathrm{util}}(u)
= - \nabla_u \Phi(u).
\end{equation} 
Note that it is still a heuristic for non-convex $\fu$ and $\fp$ such as when
using neural networks for the filter and/or the classifiers.
A related heuristic for minimax problems was proposed by \cite{Goodfellow:2014}
for learning generative models.

\begin{algorithm}[tbh] \caption{Alternating algorithm for min-diff-max} \label{alg:alternating}
%{\bf Main routine} (see Supplementary Material for a full description)\\
{\it Input}: data $\{(x_i,y_i,z_i)\}$, filter $g$, loss $l$, classifier $h$,
tradeoff coefficient $\rho$, max iteration $T$, learning rates $(\alpha_t)$\\
{\it Output}: optimal filter parameter $u$\\
{\it Begin}:
\begin{algorithmic}
\STATE{Initialize $u_1$}
\FOR{$t=1,...,T$}
\STATE{Solve (approximately) 
\begin{equation}\label{eq:inner optimization}
v_t = \arg\min_v \fp(u_t,v)\;\;\mathrm{and}\;\;w_t = \arg\min_w \fu(u_t,w),\;\;\mathrm{where}
\end{equation}
\begin{equation}\label{eq:inner optimization}
\fp(u,v)=\frac{1}{N}\sum_{i=1}^N l_p(h_p(g(x_i;u);v),y_i)\;\;\mathrm{and}\;\;
\fu(u,w)=\frac{1}{N}\sum_{i=1}^N l_u(h_u(g(x_i;u);w),z_i).
\end{equation}

}
\STATE{Compute the descent direction by
\begin{equation}\label{eq:descent direction}
q_t = \nabla_u \fp(u,v_t)-\rho\nabla_u \fu(u,w_t)
\end{equation}
}%, using
%$w_t$ and $v_t$ from the previous step}
\STATE{Perform line search along $q_t$ and update $u_{t+1} = u_t + \alpha_t\cdot q_t$}
\STATE{Exit if solution converged}
\ENDFOR
\end{algorithmic}
\end{algorithm}
%\vspace{0.1in}
The proposed optimization algorithm and supporting classes are implemented in Python
and are available on the open-source repository\footnote{\url{https://github.com/jihunhamm/MinimaxFilter}}.

\if0
Overall algorithm as per python modules?
Filter class.
Classifier class.
minimax solver.
Alternating optimizer.
\fi

%%%%%%%%%%%%%%%%%%%%%%%%%%%%%%%%%%%%%%%%%%%%%%%%%%%%%%%%%%%%%%%%%%%%%%%%%%%%%%%%%%%%%%%%%%%
\section{Noisy Minimax Filter}\label{sec:noisy minimax filter}
%%%%%%%%%%%%%%%%%%%%%%%%%%%%%%%%%%%%%%%%%%%%%%%%%%%%%%%%%%%%%%%%%%%%%%%%%%%%%%%%%%%%%%%%%%%%%

The privacy guarantee that minimax filter provides is very different from
that of differentially-private mechanisms. 
As the filter is learned from training data, its privacy guarantee for test data
is given only in expectation/probability. 
Besides, it is a deterministic mechanism which cannot provide
differential privacy.
This section presents the {\it noisy minimax filter} that
combines minimax filter with additive noise mechanism to satisfy the
differential privacy criterion. 
Two methods of combination---preprocessing and postprocessing---are 
proposed and compared.
For completeness, the definition of differential privacy is given briefly.

\subsection{Differential privacy}

A randomized algorithm that takes data $\mathcal{D}$ as input 
and outputs $\tilde{f}(\mathcal{D})$ is called $\epsilon$-differentially private if
\begin{equation} \label{eq:differential privacy}
Pr(\tilde{f}(\mathcal{D}) \in \mathcal{S})\leq
e^{\epsilon}{Pr(\tilde{f}(\mathcal{D}') \in \mathcal{S})}
\end{equation}
for all measurable $\mathcal{S} \subset \mathcal{T}$ of the output range
and for all data sets $\mathcal{D}$ and $\mathcal{D}'$ differing in a single item,
denoted by $\mathcal{D} \sim \mathcal{D}'$.
That is, even if an adversary knows the whole data set $\mathcal{D}$ except for
a single item,
she cannot infer much more about the unknown item from the output of 
the algorithm.
A well-known mechanism for turning a non-private function $f$ into a private
function $\tilde{f}$ is the perturbation by additive noise.
When an algorithm outputs a real-valued vector $f(\mathcal{D}) \in \mathbb{R}^D$, 
its global sensitivity \citep{Dwork:2006:TC} is defined as 
\begin{equation}\label{eq:sensitivity}
S(f) = \max_{\mathcal{D}\sim\mathcal{D}'} \|f(\mathcal{D}) - f(\mathcal{D}')\|
\end{equation}
where $\|\cdot\|$ is a norm such as the Euclidean norm.
An important result from \cite{Dwork:2006:TC} 
is that the perturbation by additive noise 
\begin{equation}\label{eq:mechanism}
\tilde{f}(\mathcal{D}) = f(\mathcal{D}) + \xi,
\end{equation}
where $\xi$ has the Laplace-like probability density whose scale parameter is 
proportional to $S(f)$
\begin{equation}\label{eq:sensitivity method}
P(\xi) \propto e^{-\frac{\epsilon}{S(f)}\|\xi\|},
\end{equation} 
is $\epsilon$-differentially private.

This paper considers local differential privacy \citep{Duchi:2013} of the filter
output $g(x)$, that is, perturbation is applied by each subject before $g(x)$
is released to a third party. 
Let $X=\{x_1,...,x_S\}$ be a collection of data from $S$ subjects. 
Then, $X=\{x_1,...,x_S\}$ and $X'=\{x'_1, ...,  x'_S\}$ are defined as neighbors
if $x_i=x'_i$ for all $i=1,...,S$ except for some $j \in 1,...,S$.
For this subject, $x_j$ and $x_j'$ can be any two samples from the common feature space $\mathcal{X}$ of all subjects.
Consequently,
a randomized filter $\tilde{g}(\cdot)$ is $\epsilon$-differentially private if 
for all $x,x' \in \mathcal{X}$ and all measurable $\mathcal{S} \subset \mathcal{T}$
of the output range,
\begin{equation}
Pr(\tilde{g}(x)\in S) \leq e^{\epsilon} Pr(\tilde{g}(x')\in S).
\end{equation}
To use additive noise mechanism (\ref{eq:mechanism}), the sensitivity
(\ref{eq:sensitivity}) of the output $g(\cdot)$ needs to be determined: 
\begin{equation}
S(g) = \sup_{x,x' \in \mathcal{X}} \|g(x)-g(x')\|,
\end{equation}
which is finite if $\mathcal{X}$ is compact and $g(\cdot)$ is continuous.
\if0
Care must be taken when computing the sensitivity $S(g)$, since 
the filter $g(\cdot)$ is learned from training data and is dependent on them. 
This breaks the differential privacy guarantee if one simply uses a trained
$g(\cdot)$ to compute the sensitivity.
To avoid this problem without making more assumptions on $\mathcal{X}$ or $g(\cdot)$,
\fi
If $\mathcal{X}$ is unbounded, one can directly bound the diameter of $g(\mathcal{X})$
by {\it bounding functions}. 
Examples of the bounding function $b: \mathbb{R}^D \to \mathbb{R}^D$ are
\begin{enumerate}\setlength{\itemsep}{-2pt}
\item Hard-bounding by clipping: $b(h) = \min \{1/a,\;1/\|h\|\}\cdot h$ for some $a>0$,
\item Soft-bounding by squashing: $b(h) = \tanh(a\|h\|)$ for some $a>0$, and
\item Normalization after clipping: $b(h) = h/\|h\|$,
\end{enumerate}
where $h = g(x)$ is the filter output of a sample $x\in\mathcal{X}$. 
Note that these functions enforce the sensitivity $S(b(g))$ to be at most 2, 
regardless of $\mathcal{X}$ or $g(\cdot)$. 
The threshold $a$ can be determined from the training data.

% if |h|>a:  1/|h|*h -> norm= 1.    if |h|<a, h/a -> norm<1.

%, by clipping (hard-bound), 
%squashing (soft-bound) or normalizing data. The latter is the simplest and can
%be used when the norm $\|x\|$ of features does not contain much information.
%May have impact on the performance of target classification.
%Another method is to make the value of each dimension $g(x)=[g_i(x),...,g_d(x)]$
%bounded, i.e., $|g_i(x)| \leq 1$. This can be achieved by, e.g., adding a layer of 
%$\tanh$ activation nodes in the output of the filter $g(x)$. 

\subsection{Preprocessing vs postprocessing}\label{sec:pre vs post}

\begin{figure}[tb]
\centering
%\includegraphics[width=1\linewidth]{preprocessing2_flat.pdf}
\includegraphics[width=0.6\linewidth]{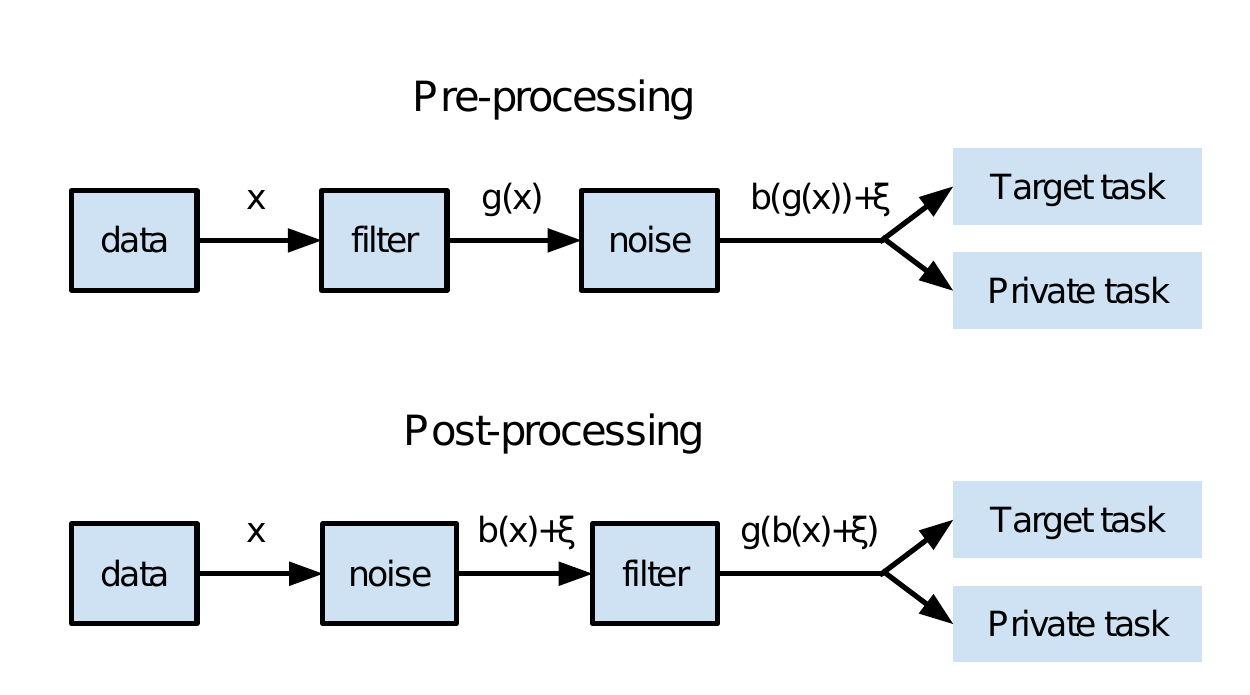}
\caption{Preprocessing and postprocessing approaches to differentially private minimax filtering.
}
\label{fig:preprocessing}
\end{figure}

Minimax filters can be made locally differentially private using the additive
noise mechanism (\ref{eq:mechanism}) in the signal chain of filtering.
The paper proposes two approaches. In the {\it preprocessing} approach,
filtering is performed first and is followed by perturbation.
In the {\it postprocessing} approach, perturbation is applied first and is followed by filtering. 
Note that preprocessing and postprocessing approaches are similar to 
output perturbation and input perturbation in \cite{Sarwate:2013:ISPM}.
Fig.~\ref{fig:preprocessing} shows the signal chains of the two approaches. 
In preprocessing, the original feature $x$ is first filtered by $g(x)$,
and then made differential private by a bounding function and perturbation
$b(g(x))+\xi$.
In postprocessing, the original feature $x$ is first made differentially
private by a bounding function and perturbation $b(x)+\xi$ followed by filtering $g(b(x)+\xi)$.
By adding an appropriate amount of noise, both approaches can be made 
$\epsilon$-differentially private regardless of data distribution.
However, when the noisy mechanism is used in conjunction with a minimax filter
which is dependent on data distribution $P(x,y,z)$, 
preprocessing and postprocessing approaches have different effects that depend
on the distribution.

A scenario when preprocessing is preferable to postprocessing is as follows. 
For the convenience of explanation, let's assume that subject identification is
the private task. Let $y(x)$ be the subject identity label of sample $x$ 
and let $z(x)$ be the target label of sample $x$ for any target task. 
Define {\it between-subject diameter} as the max distance of two samples $x, x'$
from different subjects that have the same target label:
\begin{equation}\label{eq:Sb}
S_b\triangleq\max_{x,x'\in\mathcal{X}} \|x-x'\|\;\;\;\mathrm{s. t.}\;\;\; y(x)\ne y(x'),\; z(x)=z(x').
\end{equation}
Similarly, define {\it within-subject diameter}
as the max distance of two samples $x, x'$ from the same subject that have 
different target labels:
\begin{equation}\label{eq:Sw}
S_w\triangleq\max_{x,x'\in \mathcal{X}} \|x-x'\|\;\;\;\mathrm{s. t.}\;\;\;y(x)=y(x'),\;z(x)\ne z(x').
\end{equation}
Also for the purpose of explanation, assume that the filter $g$ is an orthogonal projection
onto a lower-dimensional Euclidean space.
For a given data set $\mathcal{X}$, if the between-subject diameter is larger than 
the within-subject diameter ($S_b > S_w$) in the original feature space
(Fig.~\ref{fig:sensitivity}a), 
then minimax filtering can potentially reduce the diameter
$S(g) = \max_{x,x'} \|g(x)-g(x')\|$ significantly.
This translates to less amount of noise required to achieve the same 
$\epsilon$-privacy than the amount of noise 
required before filtering, as the data diameter has shrunk.
This will result in better utility of the preprocessing approach over
the postprocessing approach where noise is added before filtering. 
From the same reasoning, if the opposite is true ($S_w > S_b$) (Fig.~\ref{fig:sensitivity}b),
then the diameter $S(g)$ after minimax filtering does not change much, 
and the preprocessing approach may not offer much benefit over the postprocessing approach.
However, there are still other differences between the two approaches. 
This paper assumes that the training data are public information and their privacy
is not the primary concern unlike the privacy of test data. 
However, if we begin to consider the privacy of training data as well, then one should
be aware that the learned filters can leak private information,
analogous to how the PCA components can leak information about training data \citep{Chaudhuri:2012},
and that the filters also need to be sanitized before release.
The postprocessing approach makes the whole process simpler. 
In this case, after each data owner perturbs the data by herself, any subsequent 
postprocessing, whether it is the process of applying pretrained filters or
the process of training minimax filters, does not worsen
differential privacy guarantees \citep{Dwork:2014}, and therefore the postprocessing
approach is a safer choice when the data owners cannot trust the entity that collects
training data. 
%In the experiment, utility and privacy of two approaches are further compared using
%real data. 

\begin{figure}[tb]
\centering
\includegraphics[width=1\linewidth]{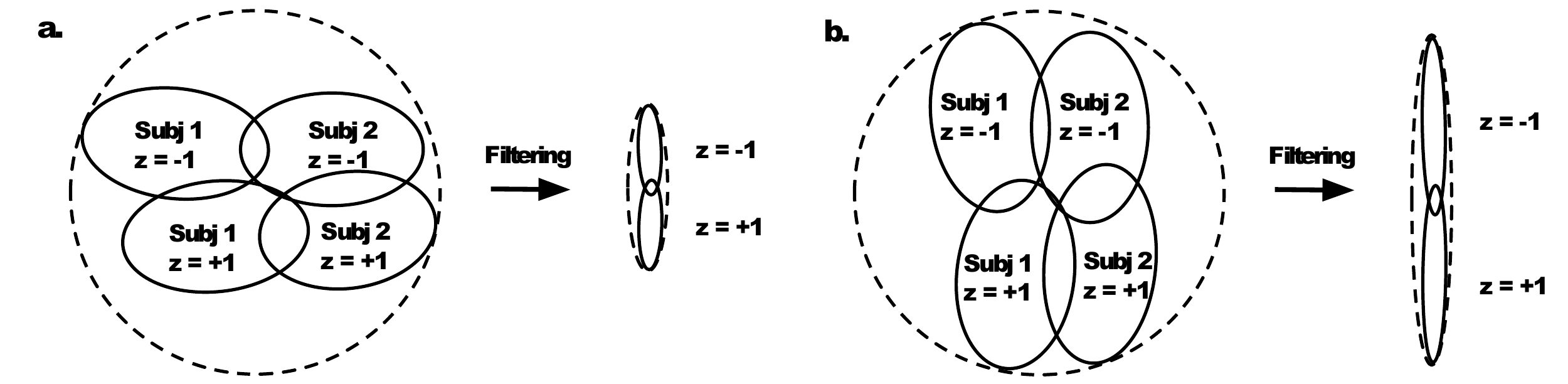}
\caption{Two example data distributions which have the same data diameter before
filtering but have different diameters after filtering.
a. an example where between-subject diameter (\ref{eq:Sb}) is large.
b. an example where within-subject diameter (\ref{eq:Sw}) is large. 
}
\label{fig:sensitivity}
\end{figure}

%\vspace{-0.2in}
%%%%%%%%%%%%%%%%%%%%%%%%%%%%%%%%%%%%%%%%%%%%%%%%%%%%%%%%%%%%%%%%%%%%%%%%%%%%%%%%%%%%%%%%%%%
\section{Experiments}\label{sec:experiments}
%%%%%%%%%%%%%%%%%%%%%%%%%%%%%%%%%%%%%%%%%%%%%%%%%%%%%%%%%%%%%%%%%%%%%%%%%%%%%%%%%%%%%%%%%%%

In this section, the algorithms proposed in the paper are evaluated using three 
real-world data sets:
face data for gender/expression classification,
speech data for emotion classification, and motion data for activity classification.
Firstly, minimax filters are compared with non-minimax methods in terms of privacy
breach vs utility as measured by accuracy of private and target tasks classifiers 
on test data.
Secondly, noisy minimax filters are tested under various conditions using
the same data sets.

\subsection{Methods}
%\subsubsection*{Filters}
{\bf Filters}.
The following minimax and non-minimax filters are compared.
%\vspace{-0.10in}
\begin{itemize} \setlength{\itemsep}{-4pt}
\item Rand: random subspace projection with $g(x;U)=U^Tx$, 
where $U$ is a random full rank $D\times d$ matrix.
\item PCA: principal component analysis with $g(x;U)=U^Tx$, where $U$ is the
 eigenvectors corresponding to $d$ largest
eigenvalues of $\mathrm{Cov}(x)$.
\item PPLS: private partial least squares, using Algorithm~1 from \cite{Enev:2012}.
\item DDD: discriminately decreasing discriminability (DDD) from
 \cite{Whitehill:2012} with a mask-type filter from the code\footnote{\url{http://mplab.ucsd.edu/\~jake}}.
\item Minimax 1: linear filter $g(x;U)=U^Tx$ where $U$ is computed from
Alg.~(\ref{alg:alternating}).
\item Minimax 2: nonlinear filter $g(x)$ from a two-layer sigmoid neural network
 with of hidden nodes of $20$ and $10$, computed from Alg.~\ref{alg:alternating}.
\end{itemize}
%\vspace{-0.10in}
{\it Remarks}. DDD requires analytical solutions to eigenvalue problems which
are unavailable for multiclass problems, and is used only in the binary problem with 
the face database.
Also, DDD uses a mask-type filter in the codes, and the dimension $d$ is same as the image size.
The dimension $d$ is also irrelevant to nonlinear Minimax filter 2 since it does not
use linear dimensionality reduction.
The nonlinear filter is pretrained as a stacked denoising autoencoders
\citep{Vincent:2008:ICML} followed by supervised backpropagation with the target task. 
%\subsubsection{Classifier/loss}
{\bf Classifier/loss}. 
For all experiments, binary or multinomial logistic regression is used a classifier
for both utility and privacy risks, where
the loss $l(h(g(x;u);v),y) $ is the negative log-likelihood with regularization:
\begin{equation} \label{eq:softmax}
l = -v(y)^Tg(x;u) + \log(\sum_{k=1}^K e^{v(k)^Tg(x;u)}) + \frac{\lambda}{2} \sum_{k=1}^K \|v(k)\|^2
\end{equation}
where $K$ is the number of classes.
%Prediction for a test sample $x$ is made by $\arg\max_k v_k'g(x)$.
%For inner optimization (\ref{eq:inner optimization}) in Alg.~\ref{alg:alternating},
%{L-BFGS} is used with the 1st and 2nd-order derivatives 
%$\nabla_u {l}, \nabla_v {l}, \frac{d^2l}{dudv}$ of the log likelihood.
%\subsubsection{Parameters}
%{\bf Parameters}.
The regularization coefficient was $\lambda=10^{-6}$ and the utility-privacy 
tradeoff coefficient was $\rho=10$. 
%The number of iterations for Minimax subroutines is set to $T_{\mathrm{aux}} = 5$ 
%for auxiliary routines (see Supplementary Material).
The main iteration in Alg.~\ref{alg:alternating} was stopped when the
progress was slow, which was between $T = 20 - 200$. 
%Other hyperparameters for Algorithm~\ref{alg:kiwiel} are in Supplementary Material.

%\subsubsection{Training/test set}
%{\bf Training/test set}:
%Training/test sets are split differently for each data set, which will 
%be explained separately. Only the test-set accuracy is reported.

\subsection{Data sets}

%\subsubsection*
{\bf Gender/expression classification from face}:
The GENKI database \citep{Whitehill:2012} consists of face images with
varying poses and facial expressions. The original data set is used unchanged, 
which has $N=1740$ training images (50\% male and 50\% female; 
50\% smile and 50\% non-smile).
The test set has 100 images (50 males and 50 females; 50 smiling and
50 non-smiling) not overlapping with the training set.
The dimensionality of the original data is $D=256$, and the filters are tested with
$d=10, 20, 50, 100$. 
The data set has gender and expression labels but no subject label. Consequently, 
gender classification is used as the private task and 
expression classification is used as the target task.

%\subsubsection*
\noindent
{\bf Emotion classification from speech}:
The ENTERFACE database \citep{Martin:2006} is an audiovisual emotion database of
43 speakers from 14 nations reading predefined English sentences in six induced emotions.
From the raw speech signals sampled in 48 KHz, MFCC coefficients are 
computed using 20 ms windows with 50\% overlap and 13 Mel-frequency bands.
%\footnote{\url{http://www.enterface.net/enterface05/main.php?frame=emotion}}
The mean, max, min, and standard deviation of the MFCC coefficients over the duration 
of each sentence are computed, resulting in $N=427$ samples of $D=52$ dimensional feature
vectors from $S=43$ subjects.
Each subject's samples are randomly split to generate training (80\%) and 
test (20\%) sets. Average test accuracy over 10 such trials is reported.
Filters are tested with $d=10, 20, 30, 40$. 
The target task is the binary classification of `happy' and `non-happy'
emotions from speech, and the privacy task is the multiclass $(S=43)$ subject
classification.

%\subsubsection*
\noindent
{\bf Activity classification from motion}:
The UCI Human Activity Recognition (HAR) data set \citep{Anguita:2012}
is a collection of motion sensor data on a smartphone by 30 subjects
performing six activities ({\it walking}, {\it walking upstairs}, 
{\it walking downstairs}, {\it sitting}, {\it standing}, {\it laying}).
%The raw sensory signals are tri-axial linear acceleration and tri-axial angular velocity
%at 50Hz, and are pre-processed by applying noise filters and then sampled in fixed-width
%sliding windows of 2.56 sec and 50\% overlap (128 readings/window). 
Various time and frequency domain variables are extracted from the signal,
resulting in $N=10299$ samples of $D=561$ dimensional features from $30$
subjects which are used unchanged.
%\footnote{\url{
%https://archive.ics.uci.edu/ml/datasets/Human+Activity+Recognition+Using+Smartphones}}.
Out of 30 subjects, 15 subjects are chosen randomly.
For each domain, each subject's samples are randomly split to generate training (50\%) and 
test (50\%) sets. 
At each trial, the subjects and the training/test sets
are randomized, and the average test accuracy over 10 such trials is reported.
Filters with dimensions $d=10, 20, 50, 100$ are used.
The target task is the multiclass ($C=6$) classification of activity,
and the privacy task is the multiclass $(S=15)$ subject classification.

\subsection{Result 1: Minimax filters}\label{sec:result 1}

Before any filter is applied, the accuracy of the target tasks with raw data is 
0.90 (GENKI), 0.84 (ENTERFACE), and 0.97 (HAR). 
On the other hand, the accuracy of the private tasks with raw data is 
0.90 (GENKI), 0.62 (ENTERFACE), and 0.70 (HAR). 
The high accuracy of the private tasks (considering the chance level accuracy of 
0.5 (GENKI), 0.02 (ENTERFACE), and 0.067 (HAR)) demonstrates that an adversary can
accurately infer private variables such as gender and identity from raw data 
if no filter is used.
A simple defense against inference attack is to perform dimensionality reduction
on the original data, such as Rand and PCA projections. 
As the dimensionality $d$ decreases from the original value $D$ towards zero, 
one can expect both the private and the target accuracy to decrease toward
the chance level. 
This trend is indeed the case with both the non-private (Rand,PCA) and the private
(PPLS,DDD,Minimax) filters used in the paper. 
Therefore these filters are evaluated at several different values of
the dimensionality $d$ to make fine-grained comparisons of utility-privacy.

\begin{figure}[tb]
\centering
%\fbox{\rule{0pt}{4in} \rule{0.9\linewidth}{0pt}}
\includegraphics[width=1\linewidth]{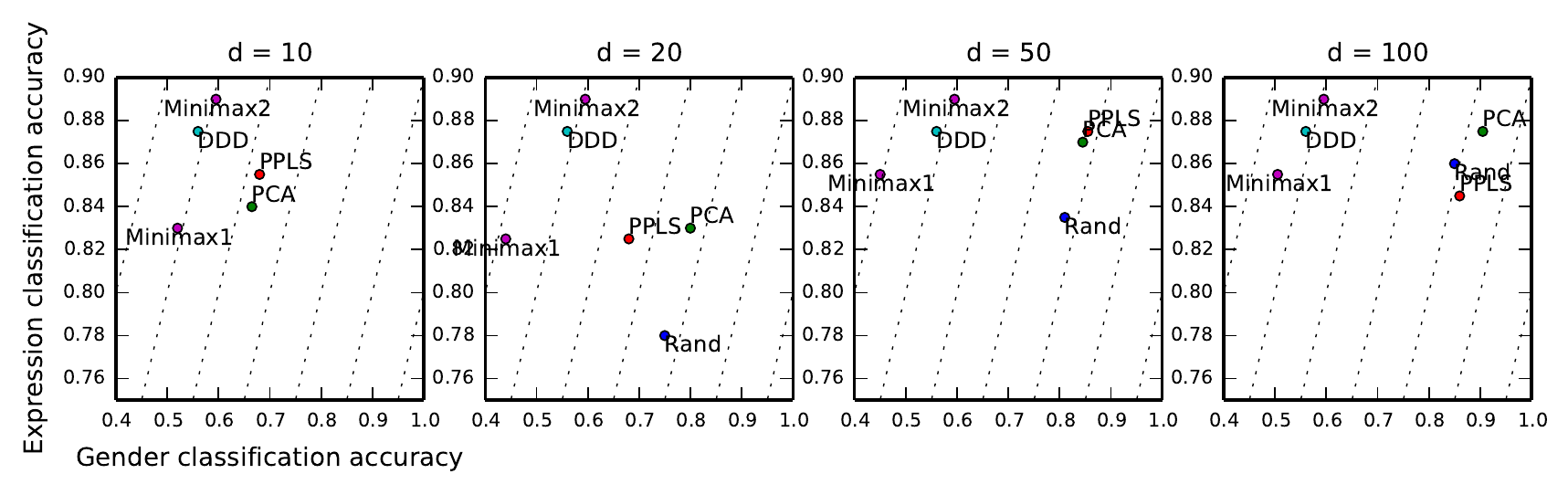}
\caption{GENKI: Expression classification vs and gender classification from faces. 
%Methods: Minimax filter 1/2 (proposed), DDD, PPLS, PCA, and Random projection.
%Dotted lines are level sets of ($c =$ accuracy of target task $-$ accuracy of private %task).
}
\label{fig:result1_genki}
\end{figure}
\begin{figure}[tb]
\centering
%\fbox{\rule{0pt}{4in} \rule{0.9\linewidth}{0pt}}
\includegraphics[width=1\linewidth]{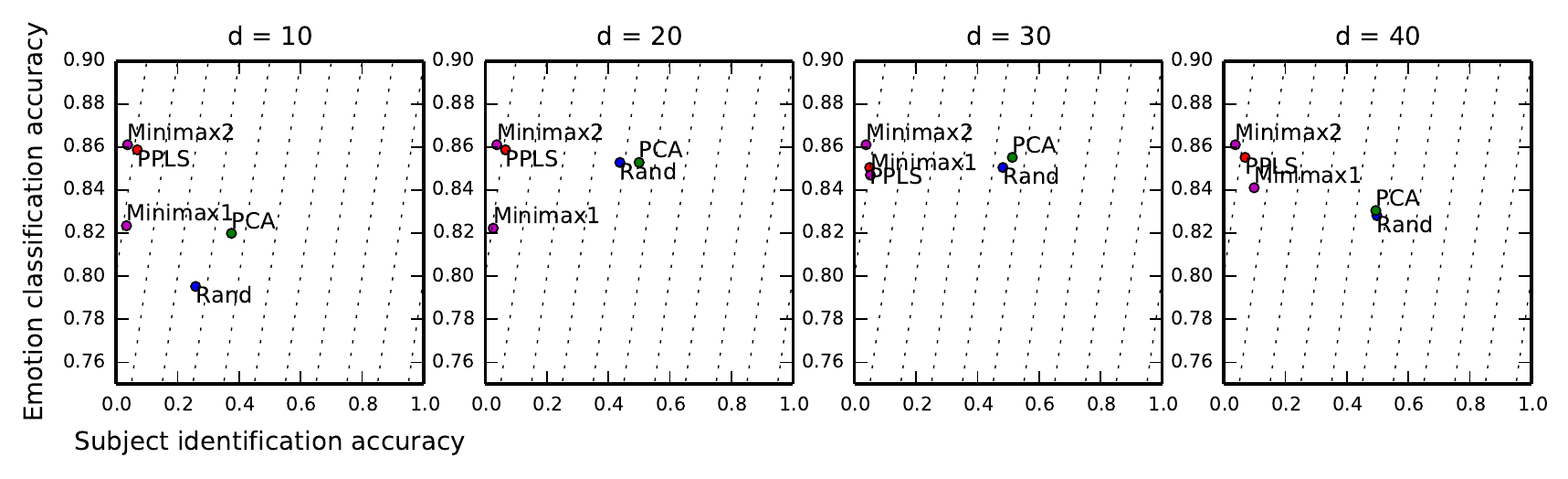}
\caption{ENTERFACE: Emotion classification vs and subject identification from speech.
%using filtered features from Minimax filter (proposed), PPLS, PCA, and Random projection.
}
\label{fig:result1_enterface}
\end{figure}

\begin{figure}[tb]
\centering
%\fbox{\rule{0pt}{1.5in} \rule{0.9\linewidth}{0pt}}
\includegraphics[width=1\linewidth]{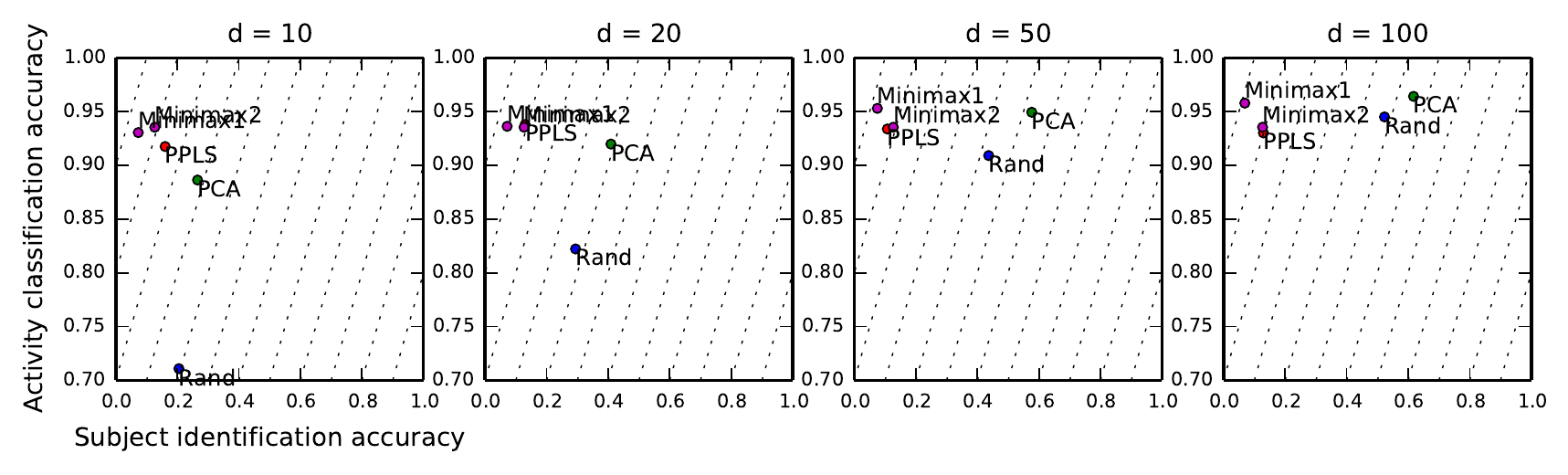}
\caption{HAR: Activity classification vs subject identification from motion.
%using filtered features from Minimax filter (proposed), PPLS, PCA, and Random projection.
}
\label{fig:result1_har}
\end{figure}

Fig.~\ref{fig:result1_genki} shows the test accuracy with GENKI. 
The dotted lines are level sets of utility-privacy tradeoff 
(i.e., target task accuracy - private task accuracy) shown for reference.
Minimax 2 achieves the best utility (i.e., most accurate expression classification) 
and Minimax 1 (linear) achieves the best privacy (i.e., least accurate 
gender classification).
For all dimensions $d$, Minimax 1 achieves the best utility-privacy compromise 
(i.e., closest to the top-left corner of the plot), with Minimax 2 and DDD
performing similarly. 
In terms of private task accuracy, Minimax 1 achieves 
almost the chance level accuracy (0.5), which implies a strong privacy preservation.
DDD comes close to Minimax 1, while another private method PPLS is not very
successful in preventing the inference of the private variable.
As expected, non-private methods Rand and PCA also do not reduce the privacy task accuracy.
As dimension $d$ increases from 10 to 100, the accuracy of both the target and 
the private tasks increase (toward the top-right corner of the plot) for 
PPLS, PCA and Rand, but the value of utility-privacy tradeoff
(i.e., target task accuracy - private task accuracy) remains relatively 
similar even though $d$ changes. Note that $d$ is irrelevant to Minimax 2 and DDD.

Fig.~\ref{fig:result1_enterface} shows the test accuracy of ENTERFACE.
Minimax 2 achieves the best utility (i.e., most accurate emotion classification) 
and the best privacy (i.e., least accurate subject classification) at the same time.
PPLS performs well in this task; its private and target task accuracy 
is close to those of Minimax 2. 
The private task accuracy of Minimax 2 is near the chance level ($1/S=0.02$)
compared to $0.4 - 0.5$ of non-private methods, suggesting that seemingly 
harmless statistics (mean, max, min, s.d. of MFCC)
are quite susceptible to identification attacks if no privacy mechanism is used.
Similar to GENKI, the accuracy of both the target and the private tasks
increases with the dimension $d$ for PCA and Rand, and the value of utility-privacy
tradeoff remains similar regardless of $d$.

Fig.~\ref{fig:result1_har} shows the test accuracy of HAR.
Minimax 1 achieves the best utility (i.e., most accurate activity recognition)
and the best privacy (i.e., least accurate subject classification),
while Minimax 2 and PPLS performs similarly well.
The private task accuracy of Minimax 1 is lower than others close to the chance level ($1/S=0.067$).
The figure also shows that motion data are susceptible ($0.2 - 0.7$) 
to identification attacks when no privacy mechanism is used.
For all dimensions $d$, Minimax 1 achieves the best compromise of all methods 
similar to previous experiments. Also the accuracy of both the target and the private
tasks roughly increases with $d$ for PCA and Rand, 
but the value of utility-privacy tradeoff remains similar.

\subsection{Result 2: Noisy minimax filters}\label{sec:result 2}

The same data sets from the previous section are used to demonstrate
the effect of noisy mechanism on minimax filters.
Four types of noisy filters are compared:
PCA-pre, PCA-post, Minimax-pre, and Minimax-post.
PCA is chosen as a non-minimax reference filter which preserves the original 
signal the best in the least mean-squared-error sense.
PCA-pre/post means that PCA is applied before/after the perturbation similarly to
Minimax-pre/post from Fig.~\ref{fig:preprocessing}.
For Minimax-pre/post, a linear filter of the same dimension $d$
as PCA-pre/post is used. 
Tests are performed for the same ranges of dimension $d$ as in Sec 6.3. 
The results for $d=20$ with all three data sets are summarized in Fig.~\ref{fig:result1_DP_all}. 
Results for different dimensions show similar trends and are summarized in Fig.~\ref{fig:result1_DP_all_full}. 
%Tests are performed for $d=20$. %different values of $d=10,20,30,40$. 
%only the case of $d=20$ is reported to avoid redundant resultsy.
%The tradeoff coefficient of $\rho=10$ is used. 
%For all classifiers, logistic regression with a regularization factor $(10^{-6})$ 
%is used throughout the tests. 
Optimization of (\ref{eq:joint goal 1}) is done similarly to the previous section.
All tests are repeated 10 times for different noise samples of 
(\ref{eq:sensitivity method}), for each of 10 random training/test splits.

\begin{figure}[tb]
\centering
\includegraphics[width=1\linewidth]{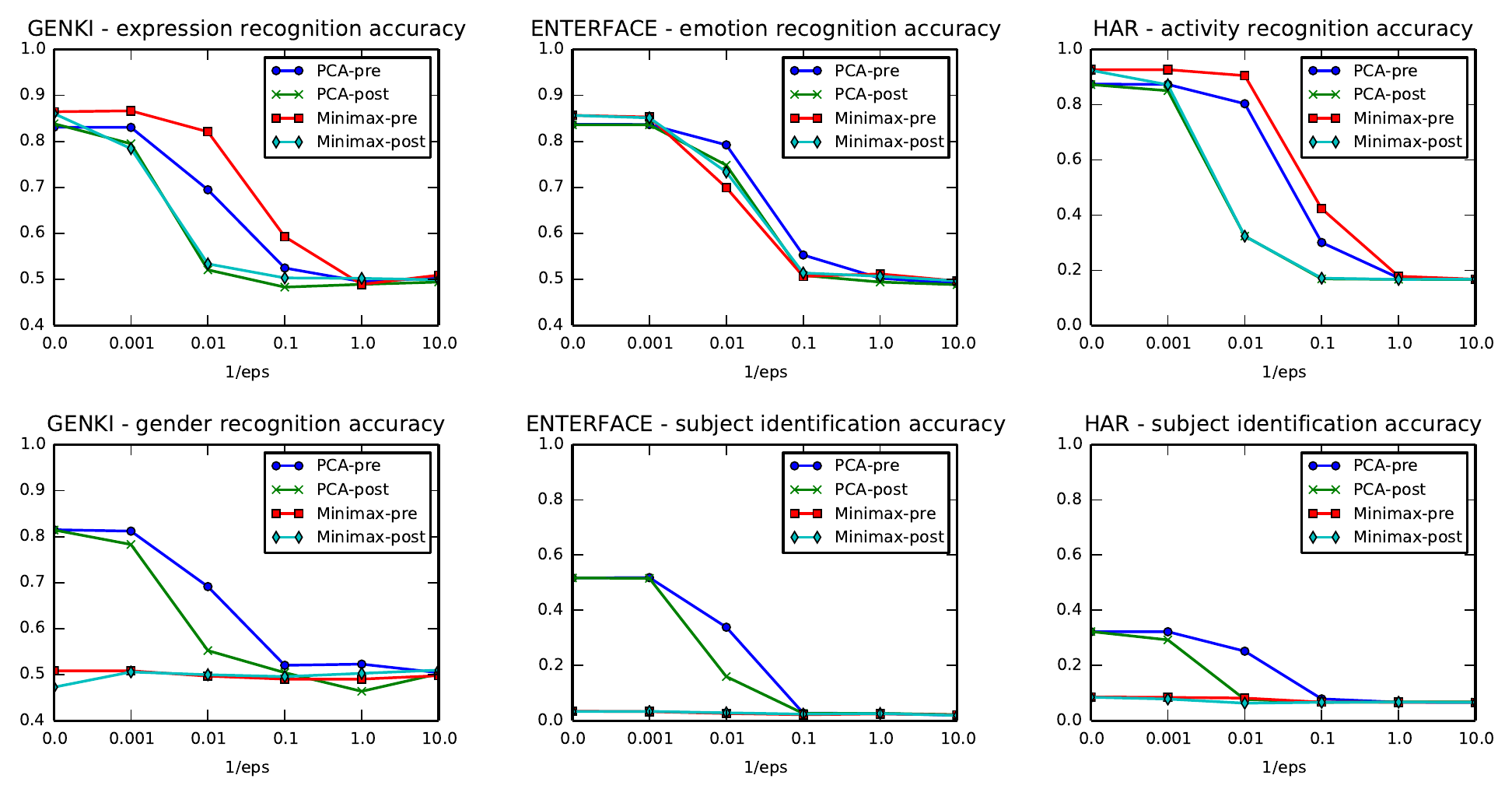}
\caption{Impact of four noisy filters
(PCA-pre/post and Minimax-pre/post) on the accuracy of target and private tasks
for three data sets (GENKI, ENTERFACE, HAR),
over the range of $\epsilon^{-1} = \{0, 10^{-3},10^{-2},10^{-1},10^0, 10^1\}$. 
Top row is the target task accuracy (higher the better) and bottom row
is the private task accuracy (lower the better.) 
Minimax-pre/post can limit the accuracy of inference attack (bottom row) 
to almost chance levels regardless of the value of $\epsilon$, while
PCA-pre/post requires a significantly high $\epsilon$ to prevent inference
attacks which also destroy the utility.}
\label{fig:result1_DP_all}
\end{figure}

Fig.~\ref{fig:result1_DP_all} shows the following results.
Firstly, within each plot, increasing the privacy level from left ($\epsilon^{-1}\mathtt{=}0$) 
to right ($\epsilon^{-1}\mathtt{=}10$) lowers the accuracy of both target and private tasks 
for all filter types and data sets, which is intuitively correct.
%With a large amount of noise, the accuracy of each task and method converges to
%a chance level that is dependent on the data set. 
Secondly, target task accuracy (top row) shows that the four filters are 
equally accurate with no noise ($\epsilon^{-1}\mathtt{=}0$), with
Minimax-pre/post slightly more accurate than PCA-pre/post. This observation
is consistent with the results in Sec.~\ref{sec:result 1}.
In GENKI and HAR, preprocessing is better than postprocessing for both
PCA and Minimax, and Minimax-pre performs the best.
In ENTERFACE, preprocessing and postprocessing approaches perform similarly,
and all four filters is perform similarly on the target task. 
This result may be ascribed to the discussion of different data distribution 
in Sec.~\ref{sec:pre vs post}.
Thirdly, and most importantly, private task accuracy (bottom row) is quite different
between Minimax-pre/post and non-minimax PCA-pre/post.
For both Minimax-pre and Minimax-post, the private task accuracy is almost
as low as the chance accuracy of each data set (0.5, 0.03, 0.07) regardless
of the noise level $\epsilon$.
This demonstrates that minimax filter can prevent
inference attacks with little help of noise.
In contrast, the non-minimax filters (PCA-pre/post) allow an adversary to infer
private variables quite accurately (0.8, 0.5, 0.3) when no noise is used.
Preventing such attacks for non-minimax filters requires a significant amount 
of additive noise 
(e.g., $\epsilon^{-1}{\geq}0.1$) which destroys the utility of data.
These results show that differentially privacy is indeed different from 
privacy against inference attacks and the combination of two methods is 
beneficial.

%%%%%%%%%%%%%%%%%%%%%%%%%%%%%%%%%%%%%%%%%%%%%%%%%%%%%%%%%%%%%%%%%%%%%%%%%%%%%%%%%%%%%%%%%%%
\section{Conclusion}\label{sec:conclusion}
%%%%%%%%%%%%%%%%%%%%%%%%%%%%%%%%%%%%%%%%%%%%%%%%%%%%%%%%%%%%%%%%%%%%%%%%%%%%%%%%%%%%%%%%%%%

This work presents a new learning-based mechanism for preventing
inference attacks on continuous and high-dimensional data.
In this mechanism, a filter transforms continuous and 
high-dimensional raw features to dimensionality-reduced representations
of data. After filtering, information on target tasks remains but 
information on identifying or sensitive attributes is removed 
which makes it difficult for an adversary to accurately infer such 
attributes from the released filtered output.
Minimax filters are designed to achieve the optimal utility-privacy tradeoff
in terms of expected risks. The paper proves that a filter learned from
empirical risks is not far from an ideal filter that is learned from expected risks
as the number of samples increases. 
This property and its dependency on the task make this mechanism
quite different from previous mechanisms, including syntactic anonymization 
and differential privacy.
Algorithms for finding minimax filters are presented and evaluated 
on real-world data sets to show its practical usages.
Experiments show that publicly available multisubject data sets are
surprisingly susceptible to subject identification attacks, 
and that even simple linear minimax filters can reduce the privacy risks
close to chance level without sacrificing target task accuracy by much.

This work also presents preprocessing and postprocessing approaches to combine
minimax privacy and differential privacy.
While differential privacy has become a popular criterion of privacy loss, 
it is not without limitations, in particular against inference attacks
as empirically demonstrated in the paper.
This leaves room for development of new mechanisms such as the noisy
minimax filter presented in the paper, which aims to achieve
high on-average utility and protection against inference attacks,
and a formal privacy guarantee to a degree. 
The results from experiments encourage further research on potential benefits of
combining different notions and mechanisms of privacy, which is left as future work.

\if0
Future work will focus on refining the proposed methods.
In particular, the postprocessing approach can potentially use the knowledge of
noise distribution to improve learning, analogous to \cite{Williams:2010}
where probabilistic estimators are learned from the noisy output of a mechanism.
\fi

%\vspace{0.3in}

\begin{figure}[tb]
\centering
\includegraphics[width=1\linewidth]{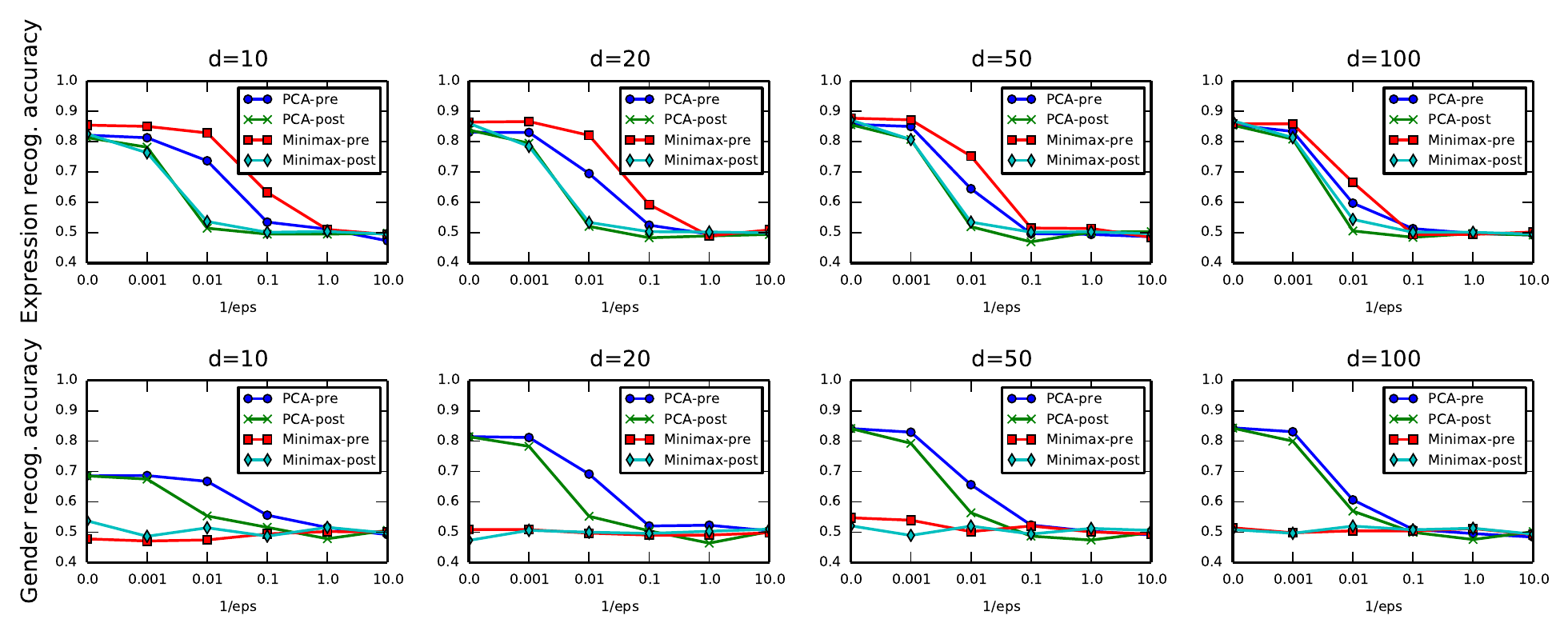}
\includegraphics[width=1\linewidth]{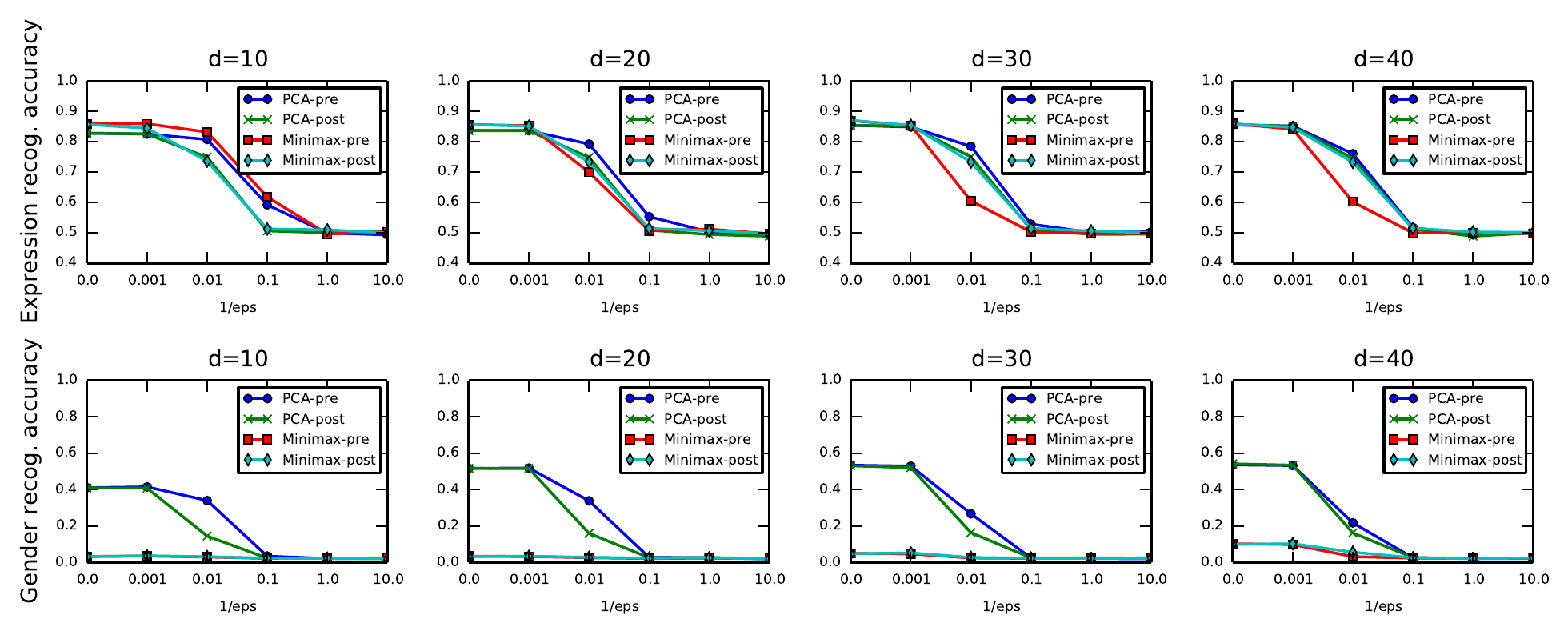}
\includegraphics[width=1\linewidth]{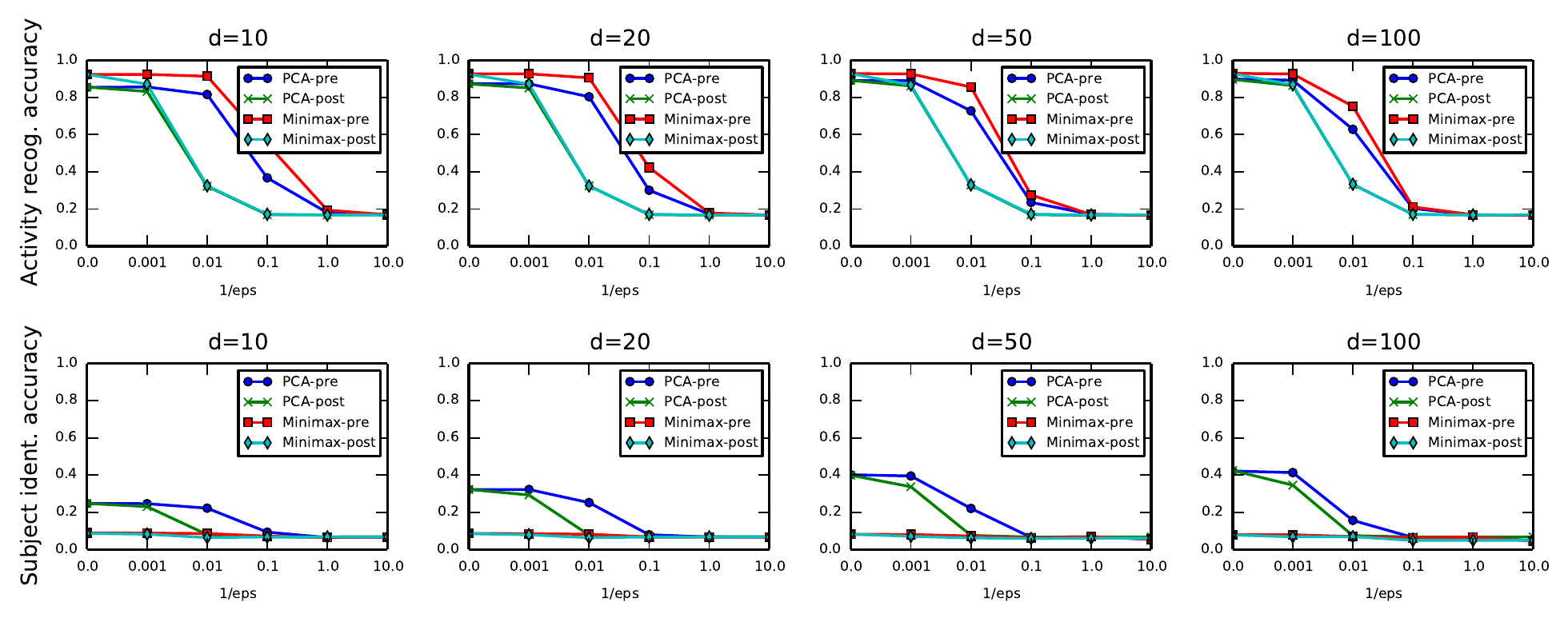}
\caption{The full result of the performance noisy filters
(PCA-pre/post and Minimax-pre/post) on the accuracy of target and private tasks
on three data sets: GENKI (1st \& 2nd row), ENTERFACE (3rd \& 4th row), HAR (5th \& 6th row).
}
\label{fig:result1_DP_all_full}
\end{figure}

\clearpage
% Acknowledgements should go at the end, before appendices and references

%\acks{This research was supported in part by funding from OFRN-C4ISR 2016
%and Google Faculty Research Award 2015.
%}


%\newpage

%\appendix
%\section*{Appendix A.}
%\label{app:theorem}

%

%\vskip 0.2in
%\bibliography{sample}

%{\small
\bibliography{jmlr15_jh}
%\bibliographystyle{abbrv}
%\bibliographystyle{icml2014}
%\bibliographystyle{abbrvnat} 
%}

\end{document}